\documentclass[journal,transmag]{IEEEtran}
\usepackage{soul}
\usepackage{url}
\usepackage[hidelinks]{hyperref}
\usepackage[utf8]{inputenc}
\usepackage[small]{caption}
\usepackage{graphicx}
\usepackage{epstopdf}
\usepackage{amsmath}
\usepackage{amssymb}
\usepackage{amsthm}
\usepackage{booktabs}
\usepackage[noend]{algpseudocode}
\usepackage{amsfonts}
\usepackage{bm}
\usepackage{subfigure}
\usepackage{cite}
\usepackage{multirow}
\usepackage{longtable}
\usepackage[ruled]{algorithm2e}
\usepackage{color}

\newtheorem{theorem}{Theorem}
\newtheorem{lemma}{Lemma}

\begin{document}
\title{Unveiling Class-Labeling Structure for Universal Domain Adaptation}

\author{\IEEEauthorblockN{Yueming Yin, 
Zhen Yang,~\IEEEmembership{Senior Member,~IEEE}, 
Xiaofu Wu, and 
Haifeng Hu}
\thanks{The authors are with School of Telecommunication and Information Engineering, Nanjing University of Posts and Telecommunications, Nanjing 210003, China.

Zhen Yang is with Key Lab of Broadband Wireless Communication and Sensor Network Technology, Ministry of Education, Nanjing University of Posts and Telecommunications, Nanjing 210003, China.

Xiaofu Wu and Haifeng Hu are with National Engineering Research Center of Communications and Networking, Nanjing University of Posts and Telecommunications, Nanjing 210003, China.

Corresponding author: Yueming Yin (email: 1018010514@njupt.edu.cn) and Zhen Yang (email: yangz@njupt.edu.cn).}}

\maketitle

\begin{abstract}
As a more practical setting for unsupervised domain adaptation, Universal Domain Adaptation (UDA) is recently introduced, where the target label set is unknown. One of the big challenges in UDA is how to determine the common label set shared by source and target domains, as there is simply no labeling available in the target domain. In this paper, we employ a probabilistic approach for locating the common label set, where each source class may come from the common label set with a probability. In particular, we propose a novel approach for evaluating the probability of each source class from the common label set, where this probability is computed by the prediction margin accumulated over the whole target domain. Then, we propose a simple universal adaptation network (S-UAN) by incorporating the probabilistic structure for the common label set. Finally, we analyse the generalization bound focusing on the common label set and explore the properties on the target risk for UDA. Extensive experiments indicate that S-UAN works well in different UDA settings and outperforms the state-of-the-art methods by large margins.
\end{abstract}


\section{Introduction}
\label{Introduction}
\IEEEPARstart{D}{omain} adaptation (DA) is a technology that allows learning models trained on available source domain (or domains) generalize to unseen target domain. To verify the performance of domain adaptation, supervised information of the target domain is required in the test stage. For training, only the target data is available if we implement domain adaptation on a new target task without any annotation. However, for classification applications, previous works supposed that the relationship between source and target label sets is known, i.e. the target label set is the same label set as source domain \cite{DBLP:conf/eccv/SaenkoKFD10,DBLP:journals/pami/DuanTX12,DBLP:conf/icml/ZhangSMW13,DBLP:journals/corr/TzengHZSD14,DBLP:conf/icml/LongC0J15,DBLP:journals/jmlr/GaninUAGLLML16,DBLP:conf/iccv/HausserFMC17,long2016unsupervised,DBLP:conf/cvpr/TzengHSD17,DBLP:conf/cvpr/SaitoWUH18,long2018conditional,DBLP:conf/iccv/CarlucciPCRB17,DBLP:conf/icml/XieZCC18,DBLP:conf/cvpr/MurezKKRK18,DBLP:conf/cvpr/VolpiMSM18,DBLP:conf/cvpr/HuKSC18,DBLP:conf/cvpr/ChenLWWC18,ganin2016domain}, a subset of source label set \cite{DBLP:conf/cvpr/CaoL0J18,DBLP:conf/cvpr/0017DLO18,cao2018partialeccv,9108582} or a set partly intersecting with source label set \cite{saito2018open,panareda2017open}. Recently, a new domain adaptation scenario has been established in \cite{you2019universal}, called Universal Domain Adaptation. Different from previous settings, UDA is a general scenario of domain adaptation, where new, unknown classes can be present and some known classes can not be found in the target domain. To solve UDA problems, a universal adaptation network (UAN) \cite{you2019universal} was proposed as shown in Fig. \ref{UAN}. In UAN, a novel weighting mechanism is proposed to weight samples when implementing domain adversarial training as shown in Eq. \ref{ED}, where $w^s(\mathbf{x})$ $($or $w^t(\mathbf{x})$$)$ is defined as the probability that a source (or target) sample $\mathbf{x}$ belongs to the common label set. This method aims to pick up samples in the common label set for alignment, because only these samples can be transferred positively. When the label in the target domain is available for verification, the ideal distribution of $w^s(\mathbf{x})$ $($or $w^t(\mathbf{x})$$)$ is shown as Fig. \ref{ideal}. In this case, $w^s(\mathbf{x})$ $($or $w^t(\mathbf{x})$$)$ is always equal to 1 if $\mathbf{x}$ belongs to the common label set, while $w^s(\mathbf{x})$ $($or $w^t(\mathbf{x})$$)$ is always equal to 0 if $\mathbf{x}$ belongs to private label sets. Then, the universal domain adaptation degenerates to close-set domain adaptation.

\begin{equation}
\begin{aligned} 
E_{D}=&-\mathbb{E}_{\mathbf{x} \sim p} w^{s}(\mathbf{x}) \log D(F(\mathbf{x}))
\\ &-\mathbb{E}_{\mathbf{x} \sim q} w^{t}(\mathbf{x}) \log (1-D(F(\mathbf{x}))),
\label{ED}
\end{aligned}
\end{equation}

\begin{equation}
\begin{aligned} 
w^{s}(\mathbf{x})&=\frac{H(\hat{\mathbf{y}})}{\log \left|\mathcal{C}_{s}\right|}-\hat{d}^{\prime}(\mathbf{x}) ,
\label{ws-uda}
\end{aligned}
\end{equation}

\begin{equation}
\begin{aligned} 
w^{t}(\mathbf{x})&=\hat{d}^{\prime}(\mathbf{x})-\frac{H(\hat{\mathbf{y}})}{\log \left|\mathcal{C}_{s}\right|}.
\label{wt-uda}
\end{aligned}
\end{equation}

In the weighting mechanism of UAN, a non-adversarial domain classifier $D'$ is designed specifically for calculating $w^s(\mathbf{x})$ (Eq. \ref{ws-uda}) and $w^t(\mathbf{x})$ (Eq. \ref{wt-uda}), where $\hat{d}^{\prime}(\mathbf{x})$ is the probability that $\mathbf{x}$ belongs to the source domain estimated by the independent binary classifier $D'$. However, the use of $D'$ is still not sufficient to distinguish between the common and private label sets in the source domain considerably as shown in Fig. \ref{w-UAN}. In our weighting mechanism, the class-labeling structure is fully utilized to capture discriminative information, which shows a conspicuous difference between common and private label sets as shown in Fig. \ref{w-ours}.

\begin{figure}
	\centering
	\includegraphics[width=0.8\linewidth]{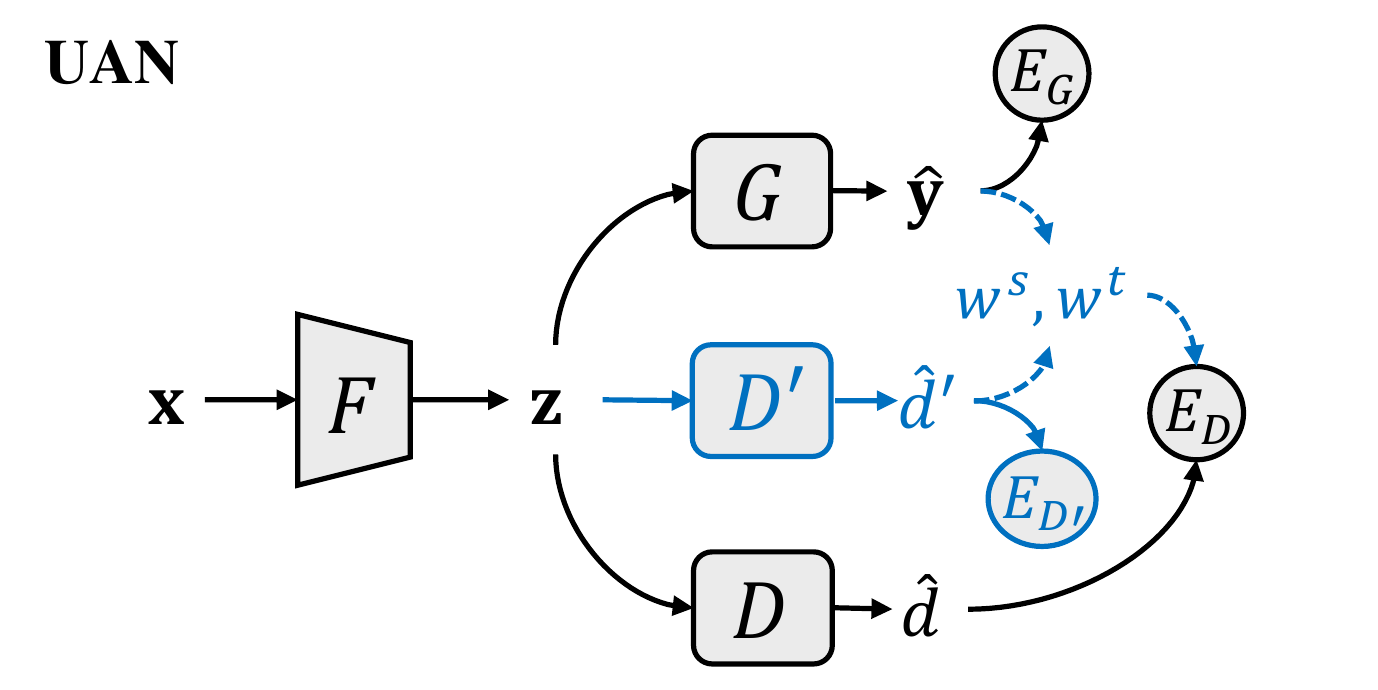}
	\caption{The training phases of the UAN.}
	\label{UAN}
\end{figure}

By empirically computing the distribution of $w^s(\mathbf{x})$ $($or $w^t(\mathbf{x})$$)$ in UAN, we find that source label set cannot be well separated by its weighting mechanism as shown in Fig. \ref{w-UAN}. There is a confusing hypothesis in their transferability criterion \cite{you2019universal}: in the common label set, the prediction entropy of samples from source domain becomes larger because they are influenced by the high entropy samples from target domain. Based on this hypothesis, the entropy of source prediction $H(\hat{\mathbf{y}})$ is used to weight samples in Eq. \ref{ws-uda}. Actually, during the training process, it is ensured that each source sample is confident in its prediction. Even if the domain adversarial training encourages features to become similar between two domains, feature extractor prefers to move the target distribution to the source ones, since this is the most effective way to minimize classification loss and maximize the domain loss simultaneously. The behavior shown in Fig. \ref{behavior} is the basic assumption of domain adaptation \cite{GaninL15}. This results in a smaller prediction entropy of target samples, rather than a larger prediction entropy of source samples \cite{you2019universal}.

\begin{figure}
	\centering
	\subfigure[Ideal]{
		\begin{minipage}[c]{0.24\textwidth}
			\centering
			\includegraphics[width=1\linewidth]{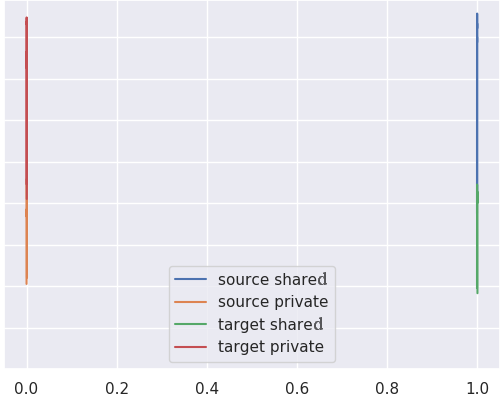}
			\label{ideal}
		\end{minipage}}
	\subfigure[UAN]{
		\begin{minipage}[c]{0.23\textwidth}
			\centering
			\includegraphics[width=1\linewidth]{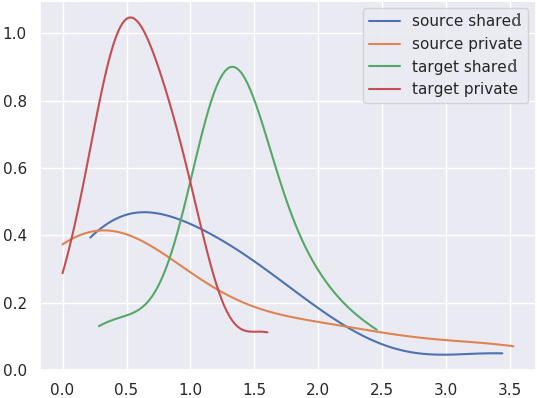}
			\label{w-UAN}
		\end{minipage}}
	\subfigure[Ours]{
		\begin{minipage}[c]{0.23\textwidth}
			\centering
			\includegraphics[width=1\linewidth]{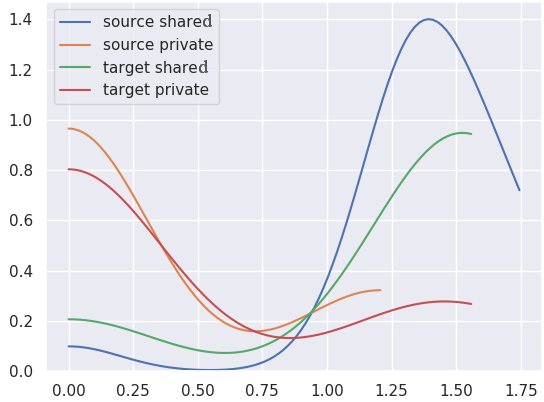}
			\label{w-ours}
		\end{minipage}}
	\caption{\textbf{Probability density distribution} of $w^s(\mathbf{x})$ $($or $w^t(\mathbf{x})$$)$ on the following groups: (1) 'source shared': source samples in the common label set (blue); (2) 'source private': source samples in the private label set (orange); (3) 'target shared': target samples in the common label set (green); (4) 'target private': target samples in the private label set (red).}
\end{figure}

\begin{figure}
	\centering
	\subfigure[High cost]{
		\begin{minipage}[c]{0.23\textwidth}
			\centering
			\includegraphics[width=0.9\linewidth]{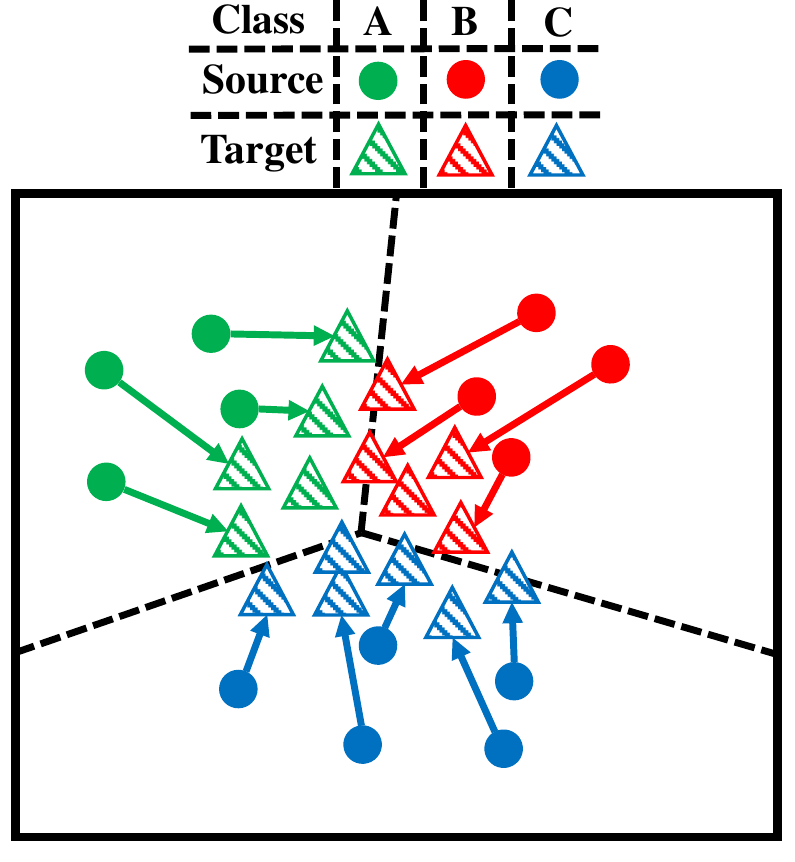}
			\label{highcost}
		\end{minipage}}
	\subfigure[Low cost]{
		\begin{minipage}[c]{0.23\textwidth}
			\centering
			\includegraphics[width=0.9\linewidth]{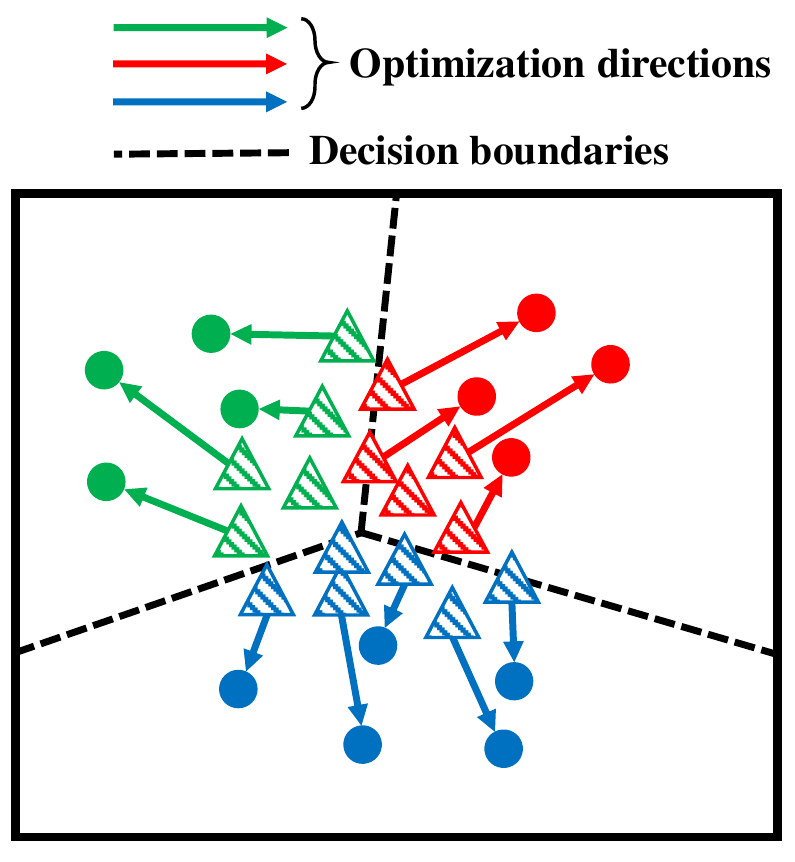}
			\label{lowcost}
		\end{minipage}}
	\caption{Two kinds of \textbf{behaviors} of a feature extractor when minimizing source label loss and maximizing domain adversarial loss simultaneously.}
	\label{behavior}
\end{figure}

According to both theoretical analysis and experimental results, the prediction entropy of samples from the source domain is not discriminative enough to find out the common label set. \textbf{Essentially}, it is the \textbf{target data} that determines the common and private classes in the \textbf{source domain}. In this paper, we propose a novel approach without use of the additional non-adversarial domain classifier. Instead, we exploit the class-labeling structure between source and target domains, and we propose a simple but effective way to distinguish between common and private classes in the source domain: the prediction margin of target samples as illustrated in Fig. \ref{motivation}. The prediction margin is defined as the largest prediction probability minus the second largest prediction probability as shown in Fig. \ref{motivation}. This margin measures the confidence in assigning a target sample to its pseudo-label, which has the largest prediction probability. \textbf{From another perspective}, when one class is predicted as the pseudo-label often with a large margin, this class is considered common in two domains. On the contrary, when one class is predicted with a small margin, this class is considered private to the target domain.

\begin{figure}
	\centering
	\includegraphics[width=1\linewidth]{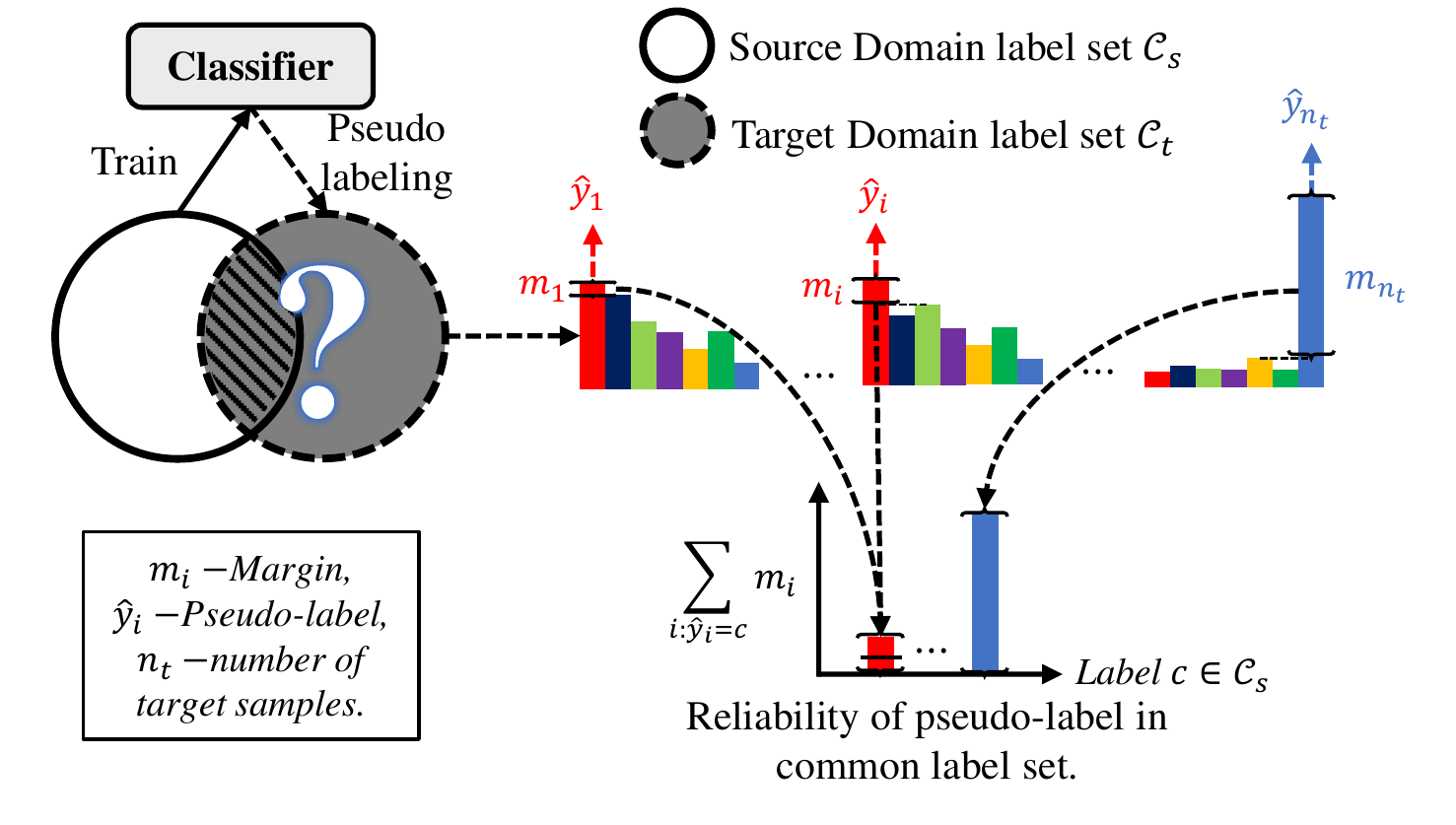}
	\caption{An illustration of Class-Labeling Structure. We employ the prediction margin as the reliability of pseudo-label in common label set. }
	\label{motivation}
\end{figure}

Based on this hypothesis, we propose a \emph{margin vector} to average the empirical margin along each pseudo-class. Each dimension of the \emph{margin vector} is the empirical margin averaged over the target samples with the same pseudo-label. Through \emph{margin vector}, one can easily evaluate the reliability of a known class appearing in the current data distribution. Specifically, in UDA setting, \emph{margin vector} integrates the reliability of each known class belonging to the common label set. These reliabilities can be used to weight source classes when implementing domain adversarial training. Experimental results show that this class-wise weighting mechanism is more practical and effective than the sample-wise weighting mechanism in \cite{you2019universal}. \textbf{Actually}, samples in the same class should share a common weight, since the probability of a class belonging to the common label set can be estimated.

The empirical margin can only be calculated after all the target data have been predicted. To this end, we propose a \emph{target margin register} (TMR) that allows the \emph{margin vector} to be updated during batch iteration. After an epoch of training, the \emph{margin vector} stored in the TMR is equal to the empirical \emph{margin vector} of all the target data. Due to lack of labeling for target samples, we use the highest probability of prediction to weight target samples. The distribution of our weighting mechanism is shown in Fig. \ref{w-ours}, where samples in the common and private label set are separated considerably.

The main contributions of this paper can be summarized as:
\begin{enumerate}
\item We propose a novel approach for evaluating the probability of each source class as from the common label set, where this probability is computed by the prediction margin accumulated over the whole target domain. This weighting mechanism requires no additional non-adversarial domain classifier, but estimates common label set efficiently.
\item We propose a simple universal adaptation network (S-UAN) containing only a feature extractor, a classifier, a domain classifier and two losses, but achieves a widely improvement on four benchmarks compared with the state-of-the-art domain adaptation methods.
\item We derive the generalization bound for UDA focusing on the common label set, and verify the properties of the generalization bound by analyzing experimental target risk.
\end{enumerate}

\section{Our Approach}
The main contribution of our work is a simple but effective solution of universal domain adaptation, where new, unknown classes can be present and some known classes may not exist in the target domain. We propose a simple universal adaptation network to solve this problem. Meanwhile, we analyse the generalization bound for UDA focusing on the common label set.

\subsection{Symbolization}
In typical domain adaptation scenario, one source domain $\mathcal{D}_{s}=\left\{\left(\mathbf{x}_{i}^{s}, \mathbf{y}_{i}^{s}\right)\right\}$ consists of $n_s$ labeled samples and one target domain $\mathcal{D}_{t}=\left\{\left(\mathbf{x}_{i}^{t}\right)\right\}$ consists of $n_t$ unlabeled samples are available. The source samples are drawn from the source distribution $p$ and the target samples from the target distribution $q$. Let $\mathcal{C}_s$ denotes the source label set and $\mathcal{C}_t$ denotes the target label set. $\mathcal{C}=\mathcal{C}_{s} \cap \mathcal{C}_{t}$ is the common label set of two domains. $\overline{\mathcal{C}}_{s}=\mathcal{C}_{s} \setminus \mathcal{C}$ and $\overline{\mathcal{C}}_{t}=\mathcal{C}_{t} \setminus \mathcal{C}$ is the label set private to the source and target domain respectively. $p_{\mathcal{C}_{s}}$, $p_{\mathcal{C}}$ and $p_{\overline{\mathcal{C}}_s}$ denote the source distribution in $\mathcal{C}_{s}$, $\mathcal{C}$ and $\overline{\mathcal{C}}_s$ respectively. Similarly, $q_{\mathcal{C}_{t}}$, $q_{\mathcal{C}}$, $q_{\overline{\mathcal{C}}_t}$ denote the target distributions in $\mathcal{C}_{t}$, $\mathcal{C}$, $\overline{\mathcal{C}}_t$ respectively. Actually, the target domain is unlabeled, and target labels are unavailable during training. The Jaccard distance \cite{you2019universal} $\xi=\frac{\left|\mathcal{C}_{s} \cap \mathcal{C}_{t}\right|}{\left|\mathcal{C}_{s} \cup \mathcal{C}_{t}\right|}$ is used to define a specific UDA scenario, where $\xi$ can be any rational number between 0 and 1. When $\xi=1$, UDA degenerates to closed set domain adaptation. A learning model should be designed to work stably with different $\xi$ even if $\xi$ is unknown. Both domain gap and category gap exists in UDA scenario, i.e. $p \neq q$ and $p_{\mathcal{C}} \neq q_{\mathcal{C}}$. The main task for UDA is to eliminate the impact of $p_{\mathcal{C}}\neq q_{\mathcal{C}}$. Meanwhile, the learning model should distinguish between target samples coming from known classes in $\mathcal{C}$ and unknown classes in $\overline{C}_{t}$. Finally, the model should be learned to minimize the target risk in $\mathcal{C}$. Additionally, we derive the theoretical generalization bound of the target risk in $\mathcal{C}$ in Section \ref{Sbound}.

\begin{figure*}[htbp]
	\centering
	\includegraphics[height=2.6in,width=0.8\linewidth]{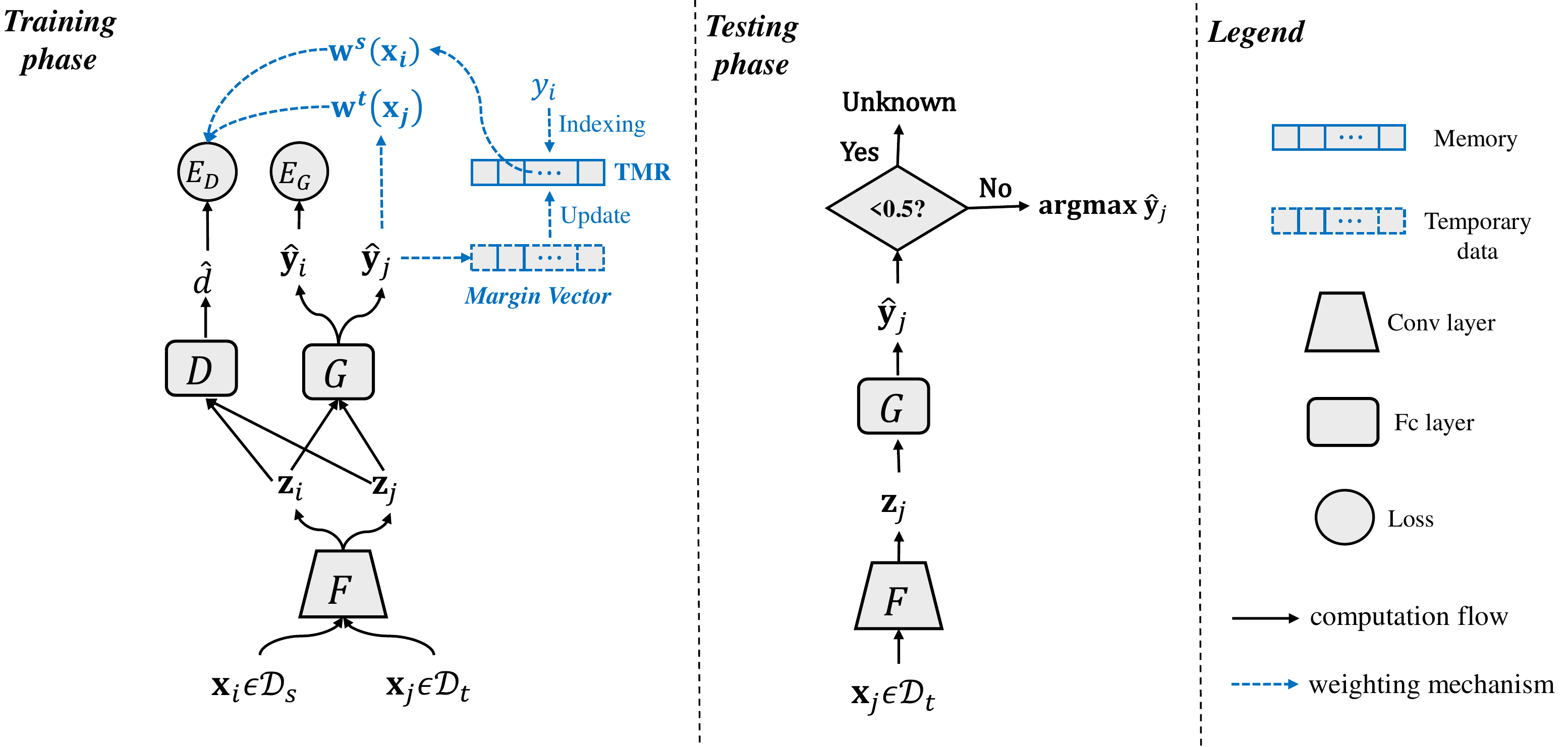}
	\caption{The training and testing phases of proposed \textbf{S-UAN}.}
	\label{model}
\end{figure*}

\subsection{Margin Vector}
We first introduce a class-wise weighting mechanism guaranteed by margin theory to UDA. The margin between features and the classification surface makes an important impact on achieving generalizable classifier. Therefore a margin theory for classification was explored by \cite{koltchinskii2002empirical,zhang2019bridging}, where the 0-1 loss for classification is replaced by the marginal loss. Here, in an \textbf{unsupervised} manner, we define the \emph{margin} of a hypothesis predictor (a classifier) $f$ at a pseudo-labeled example $\mathbf{x}$ as
\begin{equation}
\begin{aligned}
m_{f}(\mathbf{x}) &\triangleq f(\mathbf{x},\hat{y})-\max _{y \neq \hat{y}} f\left(\mathbf{x}, y\right),
\label{margin}
\end{aligned}
\end{equation}
\begin{equation}
\begin{aligned}
\hat{y}&\triangleq\mathop{\arg\max} _{y}f\left(\mathbf{x}, y\right),
\label{pseudo}
\end{aligned}
\end{equation}where $f\left(\mathbf{x}, y\right)$ is the classification probability that $\mathbf{x}$ belongs to the $y$-th class. This \emph{margin} measures the confidence in assigning an example to its pseudo-label. In particular, wrong classified samples and samples of unknown classes often have a small margin, where the classification surface intersects here.

Then, the empirical \emph{margin vector} over a data distribution $\mathcal{D}$ is defined as
\begin{equation}
\begin{aligned}
\mathcal{M}(\mathcal{D},f)\triangleq\mathbb{E}_{\mathbf{x}\in\mathcal{D}}\left[m_{f}(\{\mathbf{x}|\hat{y}=1\}),\cdots,m_{f}(\{\mathbf{x}|\hat{y}=n_f\})\right]^T,
\label{mv}
\end{aligned}
\end{equation}where $n_f$ is the number of classes defined by $f$. In this definition, the $i$-th dimension of the \emph{margin vector} is the empirical \emph{margin} of samples with the $i$-th pseudo-label.

Firstly, \emph{margin vector} can be used to find common label set for partial domain adaptation, where some of source classes do not exist in the target domain. More generally, in UDA, \emph{margin vector} can be used to extract $p_\mathcal{C}$ from $p$ by weighting source classes.

\subsection{Target Margin Register}
We first introduce a target margin register (TMR) to UDA. The TMR is formulated as a $\mathcal{C}_{s}$-dimensional vector $\mathbf{V}_{TMR}$, which is updated in each training step. The updated rule is defined as
\begin{equation}
\begin{aligned} 
\mathbf{V}^{t+1}_{TMR}=\frac{1}{t+1}(t\times \mathbf{V}^{t}_{TMR}+ \mathcal{M}(\mathcal{D}^{b}_{t},f)),
\label{updated rule}
\end{aligned}
\end{equation}
\begin{equation}
\begin{aligned} 
\mathbf{V}^{0}_{TMR}={\underbrace{[0,0,\cdots,0]}_{|\mathcal{C}_{s}|}}^T,
\label{init}
\end{aligned}
\end{equation}where $\mathcal{M}$ is the margin function defined in Eq. \ref{mv}, which outputs a $\mathcal{C}_{s}$-dimensional \emph{margin vector}. Note that $t\times \mathbf{V}^{t}_{TMR}$ is the accumulated \emph{margin vector} over previous $t$ steps. $\mathcal{D}^{b}_{t}$ denotes a batch of data sampled from $\mathcal{D}_{t}$, $f$ denotes the prediction function of the classifier, $t$ is the updating step and $t<T$, where $T$ is the maximum step set before training. $\mathbf{V}^{t}_{TMR}$ denotes the stored TMR-vector at step $t$. In Eq. \ref{updated rule}, the first term is equal to the accumulated prediction over the previous steps, the second term is the \emph{margin vector} of $\mathcal{D}^{b}_{t}$ and $f$ in the current training step, and the whole definition is equal to calculate the empirical \emph{margin vector} of all the trained batches of target samples. In S-UAN algorithn, we update $\mathbf{V}_{TMR}$ only when the source error lower than $\epsilon$ as shown in Algorithm \ref{alg}.

Each dimension of $\mathbf{V}_{TMR}$ represents the empirical \emph{margin} of the target samples predicted as corresponding source class. This margin can be directly used to weight a source sample $\mathbf{x}_{i} \in \mathcal{D}_{s}$ with its label $y_i$:
\begin{equation}
\begin{aligned}
w^s(\mathbf{x}_i)=\mathbf{V}_{TMR}[y_i],
\label{wsv}
\end{aligned}
\end{equation}where $w^s(\mathbf{x}_i)$ indicates the probability that a source sample $\mathbf{x}_i$ belongs to $\mathcal{C}$, and $\mathbf{V}_{TMR}[y_i]$ is the $y_i$-th dimension of $\mathbf{V}_{TMR}$. Unlike \cite{you2019universal}, we define the weighting mechanism in terms of class rather than individual sample. \textbf{Essentially}, samples in the same class should share a common weight, when the probability of a class belonging to $\mathcal{C}$ can be estimated. 

Note that $w^s(\mathbf{x})$ $($or $w^t(\mathbf{x})$$)$ should be further normalized as \cite{you2019universal}. In this paper, we use a modified normalization approach as 
\begin{equation}
\begin{aligned}
\overline{w}(\mathbf{x})=\frac{w(\mathbf{x})-\min_{\mathbf{x}\in \mathcal{D}^{b}} w(\mathbf{x})}{\max_{\mathbf{x}\in \mathcal{D}^{b}} w(\mathbf{x})-\min_{\mathbf{x}\in \mathcal{D}^{b}} w(\mathbf{x})},
\\
w(\mathbf{x})\leftarrow\max\left(\frac{b}{\sum_{\mathbf{x} \in \mathcal{D}^{b}}w(\mathbf{x})}\cdot \overline{w}(\mathbf{x})-w_0,0\right),
\label{norm}
\end{aligned}
\end{equation}where $\overline{w}(\mathbf{x})$ is the normalized $w(\mathbf{x})$, $d$ is the batch size, and $w_0\in\{0,1\}$ is the activation threshold, which is the only hyperparameter in our approach. Here, we suggest setting $w_0=1$ if $\xi=\frac{\left|\mathcal{C}_{s} \cap \mathcal{C}_{t}\right|}{\left|\mathcal{C}_{s} \cup \mathcal{C}_{t}\right|}$ is relatively large. And we analyze the hyperparameter $w_0$ in Section \ref{Analysis of S-UAN}.

\subsection{Simple Universal Adaptation Network}
The main challenge of UDA is to reduce the impact of $p_\mathcal{C}\neq q_\mathcal{C}$, which called domain shift in $\mathcal{C}$. To this end, let $F: \mathbb{R}^{l\times w}\rightarrow \mathbb{R}^d$ be the feature extractor, and $G: \mathbb{R}^d\rightarrow \{0,1,\cdots,|\mathcal{C}_{s}|-1\}$ be the classifier. Here, $l\times w$ and $d$ denote the dimension of input images and extracted features respectively. Input $\mathbf{x}$ from both domains are forwarded into $F$ to obtain extracted feature $\mathbf{z}=F(\mathbf{x})$, then, $\mathbf{z}$ is L2-normalized and fed into $G$ to estimate the classification probability $G^{c}(\mathbf{z})$ of $\mathbf{x}$ over the $c$-th class in $\mathcal{C}_s$. The probabilities of source input $\mathbf{x}_i$ are corrected by cross-entropy loss $\mathcal{L}_{cls}$ with source labels, whereas the highest probability of target input $\mathbf{x}_j$ can be used to calculate $w^t(\mathbf{x}_j)$ \cite{liu2019separate}, which are defined as
\begin{equation}
\begin{aligned}
E_{G}=\mathbb{E}_{(\mathbf{x}, \mathbf{y}) \sim p} L\left(\mathbf{y}, G\left(F(\mathbf{x})\right)\right),
\label{Lc}
\end{aligned}
\end{equation}
\begin{equation}
\begin{aligned}
w^t(\mathbf{x}_j)=\max _{c \in \mathcal{C}_{s}} G^c\left(F(\mathbf{x}_{j})\right),
\label{wt}
\end{aligned}
\end{equation}where $L$ is cross-entropy loss. Let $D: \mathbb{R}^d\rightarrow \{0,1\}$ be the binary classifier discriminating input feature $\mathbf{z}$ from $\mathcal{D}_s$ or $\mathcal{D}_t$. Inspired by \cite{you2019universal}, the domain classifier $D$ can be trained by weighting $w^s(\mathbf{x}_i)$ and $w^t(\mathbf{x}_j)$ as
\begin{equation}
\begin{aligned} 
E_{D}=&-\mathbb{E}_{\mathbf{x} \sim p} w^{s}(\mathbf{x}) \log D\left(F(\mathbf{x})\right)
\\ &-\mathbb{E}_{\mathbf{x} \sim q} w^{t}(\mathbf{x}) \log \left(1-D\left(F(\mathbf{x})\right)\right),
\label{Ld}
\end{aligned}
\end{equation}note that $w^{s}(\mathbf{x})$ and $w^{t}(\mathbf{x})$ are normalized as defined in Eq. \ref{norm}. With such definition, the domain alignment can be implemented in $\mathcal{C}$ as precisely as possible.

\textbf{Optimization.} The training of the whole model can be summarized as
\begin{equation}
\begin{aligned}
\max _D \min _{F, G} E_{G}-E_{D}.
\label{optimization}
\end{aligned}
\end{equation}

This optimization objective is extremely simple but particularly effective. We use the gradient reversal layer \cite{ganin2016domain} to reverse the gradient between $F$ and $D$, this allows the optimization of all modules in an end-to-end way. 

\textbf{Inference.} Finally, a target test sample $\mathbf{x}$ is assigned to either corresponding known class or the unknown class as
\begin{equation}
\begin{aligned}
y(\mathbf{x})=\left\{\begin{array}{ll}
{\arg \max}_{c\in \mathcal{C}_s} G^{c}(F(\mathbf{x})), & w^{t}(\mathbf{x})\geq 0.5 \\
\text {unknown}, & \text { otherwise }.
\end{array}\right.
\label{inference}
\end{aligned}
\end{equation}

\begin{algorithm}[h]
\caption{S-UAN Algorithm}
\label{alg}
\LinesNumbered
\KwIn{$T$: max iteration\;
$\mathcal{D}_s=\{\textbf{x}_{i},y_i\}_{i=1}^{n_s}$: source training data distribution\; 
$\mathcal{D}_t=\{ \textbf{x}_{j}\}_{j=1}^{n_t}$: target training data distribution\; 
$F$: pretrained feature extractor parameterized by $\theta _f$\; 
$G$: randomly initialized classifier parameterized by $\theta _g$\; 
$D$: randomly initialized domain classifier parameterized by $\theta _d$\; 
}
\KwOut{$\theta_f $, $\theta _g $ and $\theta _d$.}
\textbf{Training}:\\
set $t=0,\epsilon=0.1$ and initialize $V_{TMR}=V^0_{TMR}$\;
\While{$t<T$}{
\For{each batch $(\mathcal{D}^{b}_{s}, \mathcal{D}^{b}_{t})$ in $(\mathcal{D}_s, \mathcal{D}_t)$}{
	obtain extracted features on both domain: $\textbf{z}_i=F(\textbf{x}_i)$ and $\textbf{z}_j=F(\textbf{x}_j)$\;
	obtain classification results using softmax: $\hat{y}_i=G(\textbf{z}_i)$ and $\hat{y}_j=G(\textbf{z}_j)$\;
	calculate the magin $m_f(\textbf{x}_j)$ by Eq. \ref{margin} and \ref{pseudo}\;
	obtain the \emph{margin vector} $\mathcal{M}(\mathcal{D}^{b}_{t},F)$ by Eq. \ref{mv}\;
	calculate source error ($\mathbf{1}\{\cdot\}=1$ if $\{\cdot\}$ is true): $error(\mathcal{D}_{s})=\frac{1}{n_s}\sum_{i=1}^{n_s}\mathbf{1}\{G(F(\mathbf{x}_i))\neq y_i\}$\;
	\If{$error(\mathcal{D}_{s})<\epsilon$}{update TMR by Eq. \ref{updated rule}\;}
	weight $\textbf{x}_i$ using TMR-vector by Eq. \ref{wsv} and normalize by Eq. \ref{norm}\;
	weight $\textbf{x}_j$ using $\max \hat{y}_j$ by Eq. \ref{wt} and normalize by Eq. \ref{norm}\;
	calculate $\mathcal{L}_{cls}$ by Eq. \ref{Lc} and $\mathcal{L}_{d}$ by Eq. \ref{Ld}\;
	update $\theta_f $, $\theta _g $ and $\theta _d$ by optimizing Eq. \ref{optimization} using statistical gradient descent\;
	let $t\leftarrow t+1$\;}}
\textbf{Testing}:\\
Test data $\textbf{x}$ is forwarded to obtain $G^{c\in\mathcal{C}_s}(F(\textbf{x}))$\;
\If{$\max G^{c\in\mathcal{C}_s}(F(\textbf{x}))\geq 0.5$}{
	label $\textbf{x}$ with $y=\arg\max_c G^{c\in\mathcal{C}_s}(F(\textbf{x}))$\;
\Else{reject $\textbf{x}$ as an unknown class\;}}
\end{algorithm}

Unlike \cite{you2019universal}, we set the threshold as 0.5 in all the experiments. The decision condition $w^{t}(\mathbf{x}_j)\geq 0.5$ is equovlent to $\max_y f(\mathbf{x}_j,y)\geq 0.5$ according to Eq. \ref{wt}. The whole model is shown in Fig. \ref{model} and the algorithm of S-UAN is shown in Algorithm \ref{alg}.

\subsection{Generalization Bound Analysis for Universal Domain Adaptation}
\label{Sbound}
Let $\alpha=\frac{|\mathcal{C}|}{|\mathcal{C}_s|}$ and $\beta=\frac{|\mathcal{C}|}{|\mathcal{C}_t|}$ be the proportion of common classes in $\mathcal{C}_s$ and $\mathcal{C}_t$ respectively, the Jaccard distance $\xi$ can be computed as
\begin{equation}
\begin{aligned}
\xi=\frac{\alpha\beta}{(1-\alpha)(\alpha+\beta)+\alpha}.
\end{aligned}
\end{equation}

Suppose that $\mathcal{D}_{s}$ is class-balanced and the sample size of each class in $\mathcal{D}_{t}$ does not vary widely, with a large probability, the number of samples in the common label set is $\beta m$ for $m$ randomly-picked samples from $\mathcal{D}_{t}$.

\begin{lemma}
Suppose that both $\mathcal{D}_{s}$ and $\mathcal{D}_{t}$ are class-balanced. By randomly selecting the same number $m$ of samples from class-balanced $\mathcal{D}_{s}$ and $\mathcal{D}_{t}$, theorem 3.4 in \cite{KiferBG04} can be modified in the common label set:
\begin{equation}
\begin{aligned}
\begin{array}{l}
P^{(\alpha+\beta)m}\left[\left|\phi_{\mathcal{A}}\left(\mathcal{D}^m_{s}, \mathcal{D}^m_{t}\right)-\phi_{\mathcal{A}}\left(p_{\mathcal{C}}, q_{\mathcal{C}}\right)\right|>\epsilon\right] \\
\leq\left(2\alpha m\right)^{d} e^{-\alpha m \epsilon^{2} / 16}+\left(2 \beta m\right)^{d} e^{-\beta m \epsilon^{2} / 16},
\end{array}
\label{data to distribution}
\end{aligned}
\end{equation}where $P^{(\alpha+\beta)m}$ is the $(\alpha+\beta)m$ power of $P$, and it is the probability that $P$ induces over the choice of samples. $\mathcal{A}$ is a collection of subsets of some domain measure space $\phi$, i.e. $\mathcal{A}=\{\phi_\mathcal{A}:\ \phi_\mathcal{A}\in\phi\}$, and assume that the VC dimension of $\phi$ is some finite $d$. $\mathcal{D}^m_{s}$ and $\mathcal{D}^m_{t}$ are $m$ samples drawn i.i.d. from class-balanced $\mathcal{D}_{s}$ and $\mathcal{D}_{t}$.\label{lemma1}\end{lemma} Eq.\ref{data to distribution} bounds the maximum probability that, the error of measuring domain discrepancy over sampled data $\mathcal{D}^m_{s}$ and $\mathcal{D}^m_{t}$ or true distribution $p_{\mathcal{C}}$ and $q_{\mathcal{C}}$, is more than $\epsilon$. 

\begin{lemma}
Let $\gamma=\frac{\beta}{\alpha}=\frac{|\mathcal{C}_s|}{|\mathcal{C}_t|}$, $m'=\alpha m$ and $\phi_\mathcal{A}=d_{\mathcal{H} \Delta \mathcal{H}}$, which is well defined in \cite{BlitzerCKPW07}. Based on Lemma \ref{lemma1}, for any $\delta \in(0,1),$ with probability at least $1-\delta$,
\begin{equation}
\begin{aligned}
d_{\mathcal{H}\Delta \mathcal{H}}\left(p_{\mathcal{C}}, q_{\mathcal{C}}\right) &\leq \hat{d}_{\mathcal{H}\Delta \mathcal{H}}\left(\mathcal{D}_{s}, \mathcal{D}_{t}\right)
\\&+4 \sqrt{\frac{d \log (2\max\{1,\gamma\}m')+\log \left(\frac{2}{\delta}\right)}{\max\{1,\gamma\}m'}}.
\end{aligned}
\end{equation}\label{data bound}
\end{lemma}

\begin{proof}
This proof relies on Lemma \ref{lemma1} [L\ref{lemma1}], $\gamma=\frac{\beta}{\alpha}$, $m'=\alpha m$ and $\phi_\mathcal{A}=d_{\mathcal{H} \Delta \mathcal{H}}$.
\begin{equation}\nonumber
\begin{aligned}
\begin{array}{l}
P^{(\alpha+\beta)m}\left[\left|d_{\mathcal{H} \Delta \mathcal{H}}\left(\mathcal{D}^m_{s}, \mathcal{D}^m_{t}\right)-d_{\mathcal{H} \Delta \mathcal{H}}\left(p_{\mathcal{C}}, q_{\mathcal{C}}\right)\right|>\epsilon\right] \\
\leq\left(2\alpha m\right)^{d} e^{-\alpha m \epsilon^{2} / 16}+\left(2 \beta m\right)^{d} e^{-\beta m \epsilon^{2} / 16}\qquad \text{[L\ref{lemma1}]}\\
=\left(2m'\right)^{d} e^{-m'\epsilon^{2} / 16}+\left(2 \gamma m'\right)^{d} e^{-\gamma m' \epsilon^{2} / 16}\\
\leq\left(2\max\{1,\gamma\}m'\right)^{d} e^{-\max\{1,\gamma\}m'\epsilon^{2} / 16},
\end{array}
\end{aligned}
\end{equation}let $\delta=\left(2\max\{1,\gamma\}m'\right)^{d} e^{-\max\{1,\gamma\}m'\epsilon^{2} / 16}$, we have 
\begin{equation}
\begin{aligned}
\epsilon=4 \sqrt{\frac{d \log (2\max\{1,\gamma\}m')+\log \left(\frac{2}{\delta}\right)}{\max\{1,\gamma\}m'}}, 
\end{aligned}
\end{equation}then with probability at least $1-\delta$,
\begin{equation}
\begin{aligned}
|d_{\mathcal{H} \Delta \mathcal{H}}(\mathcal{D}^m_{s}, \mathcal{D}^m_{t})&-d_{\mathcal{H} \Delta \mathcal{H}}(p_{\mathcal{C}}, q_{\mathcal{C}})|\\
&\leq4 \sqrt{\frac{d \log (2\max\{1,\gamma\}m')+\log \left(\frac{2}{\delta}\right)}{\max\{1,\gamma\}m'}},
\end{aligned}
\end{equation}and $d_{\mathcal{H} \Delta \mathcal{H}}(\mathcal{D}^m_{s}, \mathcal{D}^m_{t})$ can be bounded by the empirical $\hat{d}_{\mathcal{H}\Delta \mathcal{H}}\left(\mathcal{D}_{s}, \mathcal{D}_{t}\right)$ according to the strong law of large numbers \cite{Ochs77},
\begin{equation}
\begin{aligned}
\operatorname{Pr}\left(\lim _{n \rightarrow \infty} \frac{1}{n} \sum_{i=1}^{n}d_{\mathcal{H} \Delta \mathcal{H}}((\mathcal{D}^m_{s}, \mathcal{D}^m_{t})^n)=\hat{d}_{\mathcal{H}\Delta \mathcal{H}}\left(\mathcal{D}_{s}, \mathcal{D}_{t}\right)\right)=1
\end{aligned}
\end{equation}where $(\mathcal{D}^m_{s}, \mathcal{D}^m_{t})^n$ is the $n$-th randomly-picked samples from $\mathcal{D}_{s}$ and $\mathcal{D}_{t}$. Whenever $n \rightarrow \infty$, with probability 1:
\begin{equation}
\begin{aligned}
d_{\mathcal{H}\Delta \mathcal{H}}\left(p_{\mathcal{C}}, q_{\mathcal{C}}\right) &\leq \hat{d}_{\mathcal{H}\Delta \mathcal{H}}\left(\mathcal{D}_{s}, \mathcal{D}_{t}\right)
\\&+4 \sqrt{\frac{d \log (2\max\{1,\gamma\}m')+\log \left(\frac{2}{\delta}\right)}{\max\{1,\gamma\}m'}}.
\end{aligned}
\end{equation}
\end{proof}

We assume $f:\mathcal{X}\rightarrow\{0,1,\cdots,|\mathcal{C}_s|-1\}$ is the true labeling function. For a hypothesis $h\in\mathcal{H}$, the expected risk of source and target domains in $\mathcal{C}_s$ are
\begin{equation}
\begin{aligned}
&\epsilon_{s}(h)=E_{x \sim p}\left[\left|h(x)\neq f(x)\right|\right] \\
&\epsilon_{t}(h)=E_{x \sim q}\left[\left|h(x)\neq f(x)\right|\right],
\end{aligned}
\end{equation}and based on the definition of $d_{\mathcal{H} \Delta \mathcal{H}}$ in \cite{BlitzerCKPW07}, for any hypothesis $h, h'\in \mathcal{H}$, we have
\begin{equation}
\begin{aligned}
\left|\epsilon_{S}\left(h, h^{\prime}\right)-\epsilon_{T}\left(h, h^{\prime}\right)\right| \leq \frac{1}{2} d_{\mathcal{H} \Delta \mathcal{H}}\left(p_\mathcal{C}, q_\mathcal{C}\right),
\label{dHH}
\end{aligned}
\end{equation}which has been proved in Lemma 3 in \cite{Ben-DavidBCKPV10}. Following the Definition 2 in \cite{Ben-DavidBCKPV10}, let the ideal joint hypothesis $h^{*}$ be the hypothesis which minimizes the combined risk 
\begin{equation}
\begin{aligned}
h^{*}=\underset{h \in \mathcal{H}}{\operatorname{argmin}}\ \epsilon_{S}(h)+\epsilon_{T}(h), 
\end{aligned}
\end{equation}and the combined risk of the ideal hypothesis is
\begin{equation}
\begin{aligned}
\lambda=\epsilon_{S}\left(h^{*}\right)+\epsilon_{T}\left(h^{*}\right).
\label{lambda}
\end{aligned}
\end{equation}

\begin{theorem}
Let $\mathcal{H}$ be a hypothesis space of VC dimension d. $\mathcal{D}^m_{s}$ and $\mathcal{D}^m_{t}$ are $m$ samples drawn i.i.d. from class-balanced $\mathcal{D}_{s}$ and $\mathcal{D}_{t}$ respectively. Then for any $\delta\in(0,1)$, with probability at least $1-\delta$ (over the choice of the samples), for every $h\in\mathcal{H}$:
\begin{equation}
\begin{aligned}
\epsilon_{T}(h) \leq \epsilon_{S}(h)&+\frac{1}{2} \hat{d}_{\mathcal{H} \Delta \mathcal{H}}\left(\mathcal{D}^m_{s}, \mathcal{D}^m_{t}\right)
\\&+4 \sqrt{\frac{d \log (2\max\{1,\gamma\}m')+\log \left(\frac{2}{\delta}\right)}{\max\{1,\gamma\}m'}}+\lambda.
\end{aligned}
\end{equation}\label{theorem}
\end{theorem}

\begin{proof}
The proof relies on the triangle inequality [Tr] for classification error \cite{Ben-DavidBCP06,CrammerKW08}, Eq. \ref{dHH}, Eq. \ref{lambda} and Lamma \ref{data bound} [L\ref{data bound}].
\begin{equation}\nonumber
\begin{aligned}
\epsilon_{T}(h) & \leq \epsilon_{T}\left(h^{*}\right)+\epsilon_{T}\left(h, h^{*}\right) \qquad\qquad \text{[Tr]}\\
&\leq \epsilon_{T}\left(h^{*}\right)+\epsilon_{S}\left(h, h^{*}\right)+\left|\epsilon_{T}\left(h, h^{*}\right)-\epsilon_{S}\left(h, h^{*}\right)\right|\\
& \leq \epsilon_{T}\left(h^{*}\right)+\epsilon_{S}\left(h, h^{*}\right)+\frac{1}{2} d_{\mathcal{H} \Delta \mathcal{H}}\left(p_\mathcal{C}, q_\mathcal{C}\right)\ \ \text{[Eq. \ref{dHH}]}\\
& \leq \epsilon_{T}\left(h^{*}\right)+\epsilon_{S}(h)+\epsilon_{S}\left(h^{*}\right)+\frac{1}{2} d_{\mathcal{H} \Delta \mathcal{H}}\left(p_\mathcal{C}, q_\mathcal{C}\right) \ \text{[Tr]}\\
&=\epsilon_{S}(h)+\frac{1}{2} d_{\mathcal{H} \Delta \mathcal{H}}\left(p_\mathcal{C}, q_\mathcal{C}\right)+\lambda \qquad \text{[Eq. \ref{lambda}]}\\
& \leq \epsilon_{S}(h)+\frac{1}{2} \hat{d}_{\mathcal{H} \Delta \mathcal{H}}\left(\mathcal{D}^m_{s}, \mathcal{D}^m_{t}\right)\\
&\qquad+4 \sqrt{\frac{d \log (2\max\{1,\gamma\}m')+\log \left(\frac{2}{\delta}\right)}{\max\{1,\gamma\}m'}}+\lambda \quad\text{[L\ref{data bound}]}
\end{aligned}
\end{equation}
\end{proof}

Here, we only bound the expected target risk in $\mathcal{C}_s$. The target risk in $\overline{\mathcal{C}}_t$ is much more hard to be bounded, because the distribution of unknown classes in $\overline{\mathcal{C}}_t$ could be completely different, i.e. $d_{\mathcal{H} \Delta \mathcal{H}}(p, q_{\overline{\mathcal{C}}_t})\rightarrow\infty$. From Theorem \ref{theorem}, we can find that the bound of expected target risk in $\mathcal{C}$ changes with $\gamma=\frac{|\mathcal{C}_s|}{|\mathcal{C}_t|}$ and $\alpha=\frac{|\mathcal{C}|}{|\mathcal{C}_s|}=\frac{m'}{m}$. Generally, $|\mathcal{C}_s|$ is known in domain adaptation, we have the following observations:
\begin{enumerate}
\item With fixed $|\mathcal{C}|$, if $|\mathcal{C}_t|\geq|\mathcal{C}_s|$, i.e. $\gamma=\frac{|\mathcal{C}_s|}{|\mathcal{C}_t|}\leq1$, the upper bound of expected target risk in $\mathcal{C}$ does not depend on the value of $|\mathcal{C}_t|$. If $|\mathcal{C}_t|<|\mathcal{C}_s|$, i.e. $\gamma=\frac{|\mathcal{C}_s|}{|\mathcal{C}_t|}>1$, with the increase of $|\mathcal{C}_t|$, the upper bound of expected target risk in $\mathcal{C}$ also increases when the VC dimension $d$ of the hypothesis $h$ greater than 2.
\item  With the increase of $|\mathcal{C}|$, the upper bound of expected target risk in $\mathcal{C}$ decreases when $\alpha\geq\frac{e}{2\max\{1,\gamma\}m}$, note that $0\leq|\mathcal{C}|\leq|\mathcal{C}_s|$, i.e. $\alpha=\frac{|\mathcal{C}|}{|\mathcal{C}_s|}=\frac{m'}{m}\in [0,1]$.
\end{enumerate}

\begin{table*}[h]
\centering \small
\caption{Accuracy ($\%$) of UDA tasks on \textbf{Office-31} $(\xi=0.32)$, \textbf{ImageNet-Caltech} $(\xi=0.07)$ and \textbf{VisDA2017} $(\xi=0.50)$ (Backbone: ResNet-50)}
\setlength{\tabcolsep}{3mm}{
\begin{tabular}{ccccccccccc}
\cmidrule{1-11}
\multirow {2.5}{*}{Method} & \multicolumn {7}{c}{Office-31} & \multicolumn {2}{c}{ImageNet-Calech} & \multirow {2.5} {*} {VisDA} \\
\cmidrule(lr){2-8}\cmidrule(lr){9-10} & $\mathrm{A} \rightarrow \mathrm{W}$ & $\mathrm{D} \rightarrow \mathrm{W}$ & $\mathrm{W} \rightarrow \mathrm{D}$ & $\mathrm{A} \rightarrow \mathrm{D}$ & $\mathrm{D} \rightarrow \mathrm{A}$ & $\mathrm{W} \rightarrow \mathrm{A}$ & Avg & $\mathrm{I} \rightarrow \mathrm{C}$ & $\mathrm{C} \rightarrow \mathrm{I}$ \\
\cmidrule{1-11}
ResNet & 75.94 & 89.60 & 90.91 & 80.45 & 78.83 & 81.42 & 82.86 & 70.28 & 65.14 & 52.80 \\
DANN & 80.65 & 80.94 & 88.07 & 82.67 & 74.82 & 83.54 & 81.78 & 71.37 & 66.54 & 52.94 \\
RTN & 85.70 & 87.80 & 88.91 & 82.69 & 74.64 & 83.26 & 84.18 & 71.94 & 66.15 & 53.92 \\
IWAN & 85.25 & 90.09 & 90.00 & 84.27 & 84.22 & 86.25 & 86.68 & 72.19 & 66.48 & 58.72 \\
PADA & 85.37 & 79.26 & 90.91 & 81.68 & 55.32 & 82.61 & 79.19 & 65.47 & 58.73 & 44.98 \\
ATI  & 79.38 & 92.60 & 90.08 & 88.40 & 78.85 & 81.57 & 84.48 & 71.59 & 67.36 & 54.81 \\
OSBP & 66.13 & 73.57 & 85.62 & 72.92 & 47.35 & 60.48 & 67.68 & 62.08 & 55.48 & 30.26 \\
UAN & 85.62 & 94.77 & 97.99 & 86.50 & 85.45 & 85.12 & 89.24 & 75.28 & 70.17 & 60.83 \\
\cmidrule{1-11}
S-UAN $w_0=0$ & 86.64 & 95.47 & 97.08 & 85.35 & 90.27 & 90.20 & 90.84 & $\verb|--|$ & $\verb|--|$ & 65.05 \\
S-UAN & \bf{93.16} & \bf{97.04} & \bf{98.08} & \bf{93.39} & \bf{91.23} & \bf{91.27} & \bf{94.03} &  \bf{78.63} & \bf{72.06} & \bf{65.20} \\
\cmidrule{1-11}
\label{Office-31}
\end{tabular}}
\end{table*}

\begin{table*}
\centering \small
\caption{Accuracy (\%) of UDA tasks on \textbf{Office-Home} $(\xi=0.15)$ dataset (Backbone; ResNet-50)}
\setlength{\tabcolsep}{0.5mm}{
\begin{tabular}{cccccccccccccc}
\cmidrule{1-14}
\multirow {2.5}{*}{Method} & \multicolumn {13}{c}{Office-Home} \\
\cmidrule{2-14}& $\mathrm{Ar} \rightarrow \mathrm{Cl}$ & $\mathrm{Ar} \rightarrow \mathrm{Pr}$ & $\mathrm{Ar} \rightarrow \mathrm{Rw}$ & $\mathrm{Cl} \rightarrow \mathrm{Ar}$ & $\mathrm{Cl} \rightarrow \mathrm{Pr}$ & $\mathrm{Cl} \rightarrow \mathrm{Rw}$ & $\mathrm{Pr} \rightarrow \mathrm{Ar}$ & $\mathrm{Pr} \rightarrow \mathrm{Cl}$ & $\mathrm{Pr} \rightarrow \mathrm{Rw}$ & $\mathrm{Rw} \rightarrow \mathrm{Ar}$ & $\mathrm{Rw} \rightarrow \mathrm{Cl}$ & $\mathrm{Rw} \rightarrow \mathrm{Pr}$ & Avg \\
\cmidrule{1-14}
ResNet & 59.37 & 76.58 & 87.48 & 69.86 & 71.11 & 81.66 & 73.72 & 56.30 & 86.07 & 78.68 & 59.22 & 78.59 & 73.22 \\
DANN & 56.17 & 81.72 & 86.87 & 68.67 & 73.38 & 83.76 & 69.92 & 56.84 & 85.80 & 79.41 & 57.26 & 78.26 & 73.17 \\
RTN & 50.46 & 77.80 & 86.90 & 65.12 & 73.40 & 85.07 & 67.86 & 45.23 & 85.50 & 77.20 & 55.55 & 78.79 & 70.91 \\
IWAN & 52.55 & 81.40 & 86.51 & 70.58 & 70.99 & 85.29 & 74.88 & 57.33 & 85.07 & 7.48 & 59.65 & 78.91 & 73.39 \\
PADA & 39.58 & 69.37 & 76.26 & 62.57 & 66.39 & 77.47 & 48.39 & 35.79 & 79.60 & 75.94 & 44.50 & 78.10 & 62.91 \\
ATI & 52.90 & 80.37 & 85.91 & 71.08 & 72.41 & 84.39 & 74.28 & 57.84 & 85.61 & 76.06 & 60.17 & 78.42 & 73.29 \\
OSBP & 47.75 & 60.90 & 76.78 & 59.23 & 61.58 & 74.33 & 61.67 & 44.50 & 79.31 & 70.59 & 54.95 & 75.18 & 63.90 \\
UAN & \bf{63.00} & 82.83 & 87.85 & 76.88 & 78.70 & 85.36 & 78.22 & 58.59 & 86.80 & 83.37 & \bf{63.17} & 79.43 & 77.02 \\
\cmidrule{1-14}
S-UAN & 61.91 & \bf{82.85} & \bf{91.73} & \bf{83.41} & \bf{81.78} & \bf{89.32} & \bf{81.86} & \bf{59.83} & \bf{90.51} & \bf{83.42} & 61.84 & \bf{83.45} & \bf{79.33} \\
\cmidrule{1-14}
\label{Office-Home}
\end{tabular}}
\end{table*}

\section{Experiments}
\subsection{Experimental Setup}
\subsubsection{Datasets}
\textbf{Office-31} \cite{saenko2010adapting} dataset contains 31 classes and 3 different domains. They are photos from amazon (\textbf{A}), dslr (\textbf{D}) and webcam (\textbf{W}). The 10 classes shared by \textbf{Office-31} and \textbf{Caltech-256} \cite{gong2012geodesic} are used as $\mathcal{C}$, then in alphabetical order, the next 10 classes as $\overline{\mathcal{C}}_{s},$ and the reset 11 classes as $\overline{\mathcal{C}}_{t}$. In these tasks $\xi=0.32$.

\textbf{Office-Home} \cite{venkateswara2017deep} consists of 65 categories and 4 domains: Artistic (\textbf{Ar}), Clip-Art (\textbf{Cl}), Product (\textbf{Pr}) and Real-World (\textbf{Rw}) images. Similarly, in alphabet order, the first 10 classes are used as $\mathcal{C}$ and the next 5 classes as $\overline{\mathcal{C}}_{s}$. According to the first property of Theorem \ref{theorem}, we use the rest 50 classes as $\overline{\mathcal{C}}_{t}$. In these tasks $\xi=0.15$.

\textbf{VisDA2017} \cite{peng2018visda} dataset consists of 2 very different domains: synthetic images (\textbf{Syn}) generated by game engines and real images (\textbf{Real}). Both of them contain 12 classes, and we only consider the practical \textbf{Syn}$\rightarrow$\textbf{Real} task. The first 6 classes are used as $\mathcal{C}$, the next 3 classes as $\overline{\mathcal{C}}_{s}$ and the rest as $\overline{\mathcal{C}}_{t}$. In these tasks $\xi=0.50$.

\textbf{ImageNet-Caltech} is established from \textbf{ImageNet-1K} ($\mathbf{I}$) \cite{russakovsky2015imagenet} and \textbf{Caltech-256} ($\mathbf{C}$) with 1000 and 256 classes respectively. The 84 classes shared by both domains are used as $\mathcal{C}$ and their private classes as $\overline{\mathcal{C}}_s$ and $\overline{\mathcal{C}}_t$. This task naturally belongs to the universal domain adaptation scenario due to the category gap between two datasets. There are two UDA tasks can be formed: $\mathbf{I} \rightarrow \mathbf{C}$ and $\mathbf{C} \rightarrow \mathbf{I}.$ Note that the validation sets are used as target domain to eliminate the effects of pre-training. In these tasks $\xi=0.07$.

Some examples in each dataset are shown in Fig. \ref{datasets}.

\begin{figure}
	\centering
	\includegraphics[width=1\linewidth]{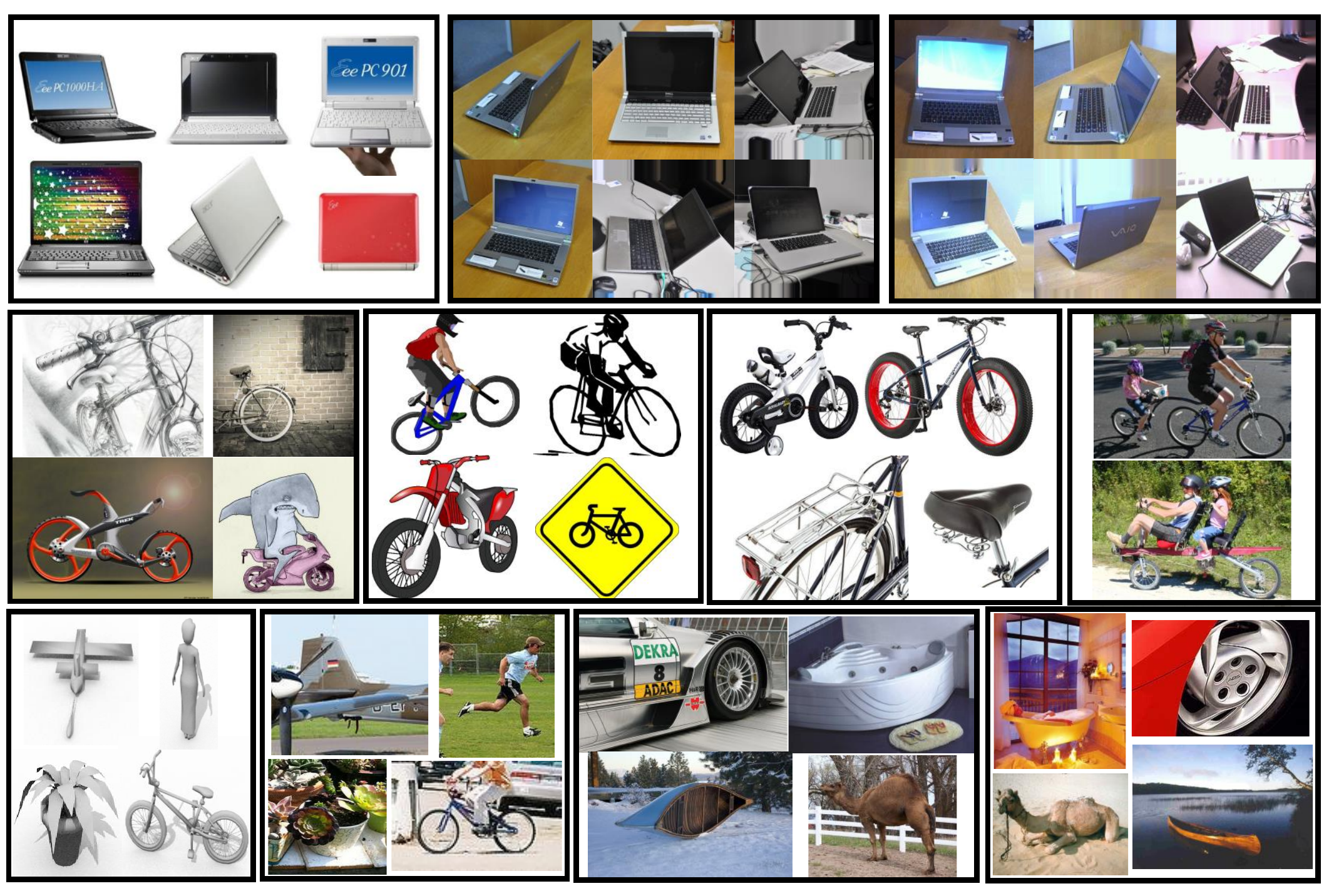}
	\caption{Some examples in following datasets: from left to right, the first line is Amazon (\textbf{A}), Dslr (\textbf{D}) and Webcam (\textbf{W}) in Office-31. The second line is Artistic (\textbf{Ar}), Clip-Art (\textbf{Cl}), Product (\textbf{Pr}) and Real-World images (\textbf{W}) in Office-Home. The third line is Synthetic (\textbf{Syn}) and Real images (\textbf{Real}) in VisDA2017, ImageNet-1K (\textbf{I}) and Caltech-256 (\textbf{C}).}
	\label{datasets}
\end{figure}

\subsubsection{Evaluation Details}
\textbf{Compared Methods}. The proposed S-UAN is compared with methods in (\textbf{1}) Source-only without domain adaptation (DA): \textbf{ResNet} \cite{he2016deep}, (\textbf{2}) close-set DA: Domain-Adversarial Neural Networks (\textbf{DANN}) \cite{ganin2016domain}, Residual Transfer Networks (\textbf{RTN}) \cite{long2016unsupervised}, (\textbf{3}) partial DA: Importance Weighted Adversarial Nets (\textbf{IWAN}) \cite{zhang2018importance}, Partial Adversarial DA (\textbf{PADA}) \cite{cao2018partialeccv}, (\textbf{4}) open set DA: Assign-and-Transform-Iteratively (\textbf{ATI}) \cite{panareda2017open}, Open Set Back-Propagation (\textbf{OSBP}) \cite{saito2018open}, (\textbf{5}) UDA: Universal Adaptation Network (\textbf{UAN}) \cite{you2019universal}. These methods are the state of the art in their respective scenarios. In the following experiments, we evaluation all these methods in the UDA scenario.

\textbf{Evaluation Protocols}. All the samples with their labels in $\overline{\mathcal{C}}_t$ are viewed as one unified "unknown" class and the final accuracy is averaged by per-class accuracy for $|\mathcal{C}|+1$ classes. We extend non-universal methods by threshold mechanism as used in \cite{you2019universal}. Only if the prediction confidence is higher than the threshold, the input image is considered as the most likely class, otherwise the image is viewed as the "unknown" class.

\textbf{Implementation Details}. We implement all the methods in PyTorch and use ResNet-50 \cite{he2016deep} as fine-tuned model, which is pre-trained on ImageNet. The hyperparameter $w_0$ is set to 1 in Office-31 and VisDA, where $\xi$ are relatively large. While $w_0$ is set to 0 in Office-Home and ImageNet-Caltech, where $\xi$ are relatively small. Particularly, in ImageNet-Caltech task, a gradually increasing threshold are used in Eq. \ref{inference}, which reaches 0.5 at the end of training.

\begin{figure}
	\centering
	\includegraphics[width=1\linewidth]{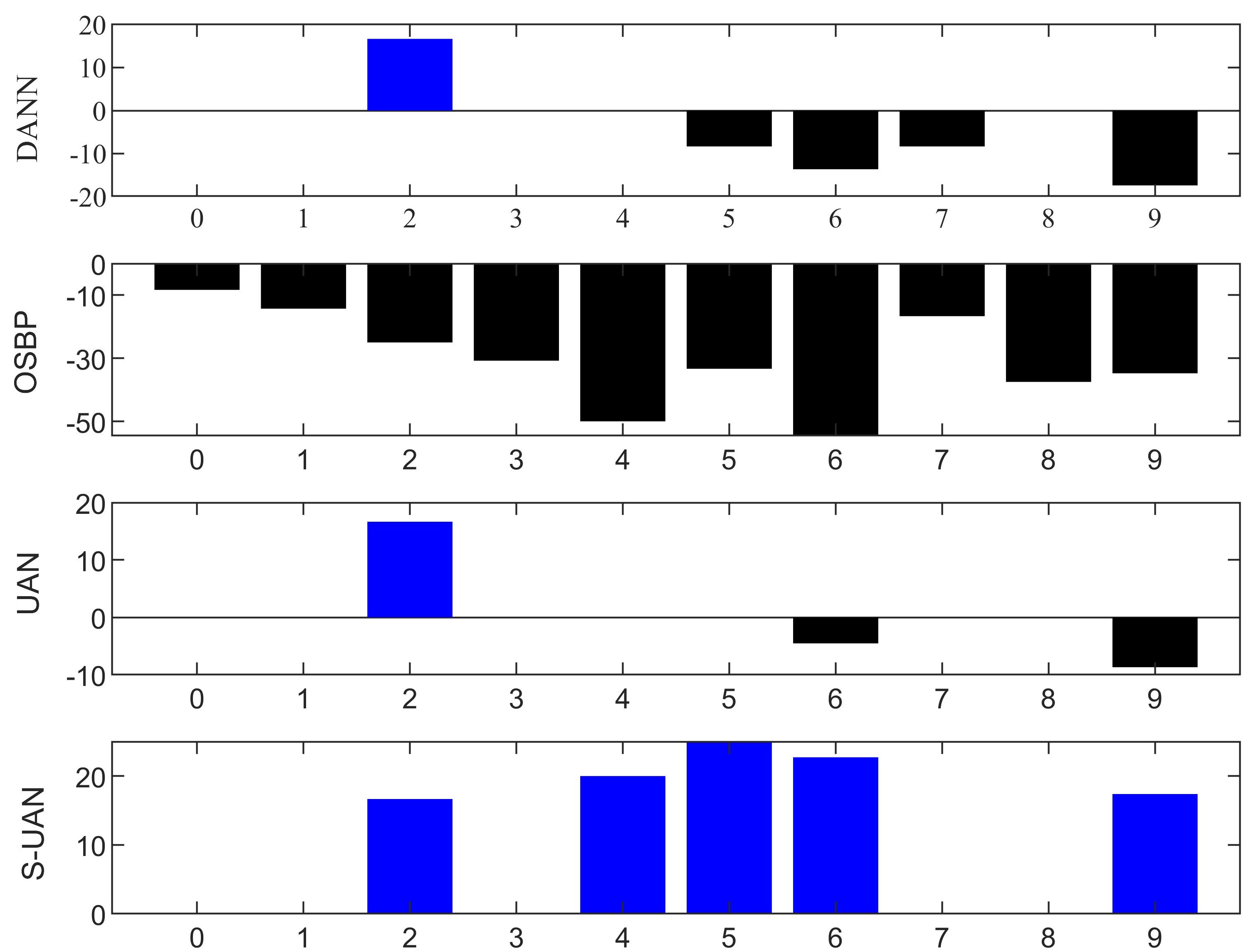}
	\caption{The \textbf{negative transfer} in UDA (task $\mathbf{A\rightarrow D}$).}
	\label{neg}
\end{figure}

\begin{figure}[h]
	\centering
	\subfigure[UAN]{
		\begin{minipage}[c]{0.23\textwidth}
			\centering
			\includegraphics[width=1\linewidth]{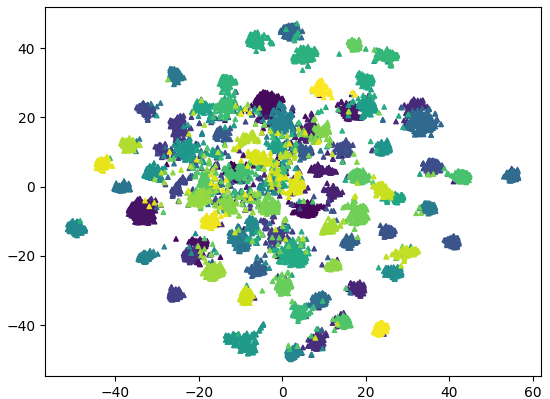}
			\label{I-C-UAN}
		\end{minipage}}
	\subfigure[S-UAN]{
		\begin{minipage}[c]{0.23\textwidth}
			\centering
			\includegraphics[width=1\linewidth]{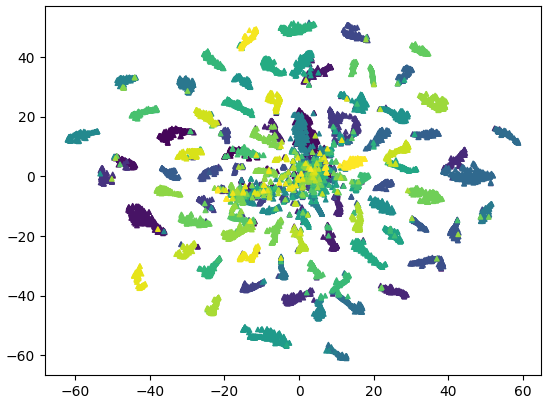}
			\label{I-C}
		\end{minipage}}
	\caption{\textbf{t-SNE visualization} on task \textbf{ImageNet-Caltech}. Each figure contains all the test samples in the common label set (84 classes). Points are colored with their true labels.}
	\label{t-SNE}
\end{figure}

\subsection{Classification Results}
\label{Classification Results}
From Tables \ref{Office-31} and \ref{Office-Home}, we can find that S-UAN outperforms all the compared methods. Particularly, we focus on the comparison with UAN \cite{you2019universal}. Although UAN avoids negative transfer in most tasks, it still has negative transfer in some tasks. For example, Fig. \ref{neg} indicates the accuracy gain of each class compared to ResNet (no adaptation) on task $A\rightarrow D$. We can see that DANN, OSBP and UAN still suffer from negative transfer. Note that the accuracy of the label set $\{0,1,3,7,8\}$ arrives $100\%$ in ResNet. Except these classes, S-UAN avoids negative transfer and achieve widely positive transfer in all the classes. This is because S-UAN makes full use of the prediction information of the target data and employs class-wise weighting mechanism in the source domain.

Another perspective to analyze the classification results is the dimension-reduction of features. In general, low-dimensional distinguishable features are easier to distinguish in higher dimensions. To this end, we visualize features extracted from trained models by t-SNE tool. In Fig. \ref{t-SNE}, we compare S-UAN and UAN on task ImageNet-Caltech, which contains 84 common classes and 1173 total classes. Overall, the distinguishability of UAN features is not as good as S-UAN, so the classification accuracy of S-UAN is higher than UDA's.

\subsection{Analysis on UDA Settings}
\textbf{Effect of Varying Size of $\overline{\mathcal{C}}_{t}$}. We explore the performance of methods mentioned in Section \ref{Classification Results} with the different sizes of $\overline{\mathcal{C}}_{t}$ on $\operatorname{task} \mathbf{A} \rightarrow \mathbf{D}$ in Office31 dataset, note that $\overline{\mathcal{C}}_{s}$ also changes correspondingly with fixed $\left|\mathcal{C}_{s} \cup \mathcal{C}_{t}\right|$ and $\xi$. As shown in Fig. \ref{ct}, S-UAN outperforms all the compared methods on all the sizes of $\overline{\mathcal{C}}_{t}$. Particularly, when $\left|\overline{\mathcal{C}}_{t}\right|=0,$ this UDA setting degenerates to the partial domain adaptation, where $\mathcal{C}_{t} \subset \mathcal{C}_{s}$, and when $\left|\overline{\mathcal{C}}_{t}\right|=21$, this UDA setting degenerates to the open set domain adaptation, where $\mathcal{C}_{s} \subset \mathcal{C}_{t},$ the performance of S-UAN is comparable to UAN's. When $\left|\overline{\mathcal{C}}_{t}\right|$ vary between 0 and 21, which are the more general settings, S-UAN performs much better than UAN. When the size of $\overline{\mathcal{C}}_{t}$ appears in the middle of 0 and 21, S-UAN outperforms other methods with the largest margin. Generally, S-UAN produces better results with all size of $\overline{\mathcal{C}}_{t}$. 

\textbf{Effect of Varying Size of $\mathcal{C}$}. We also explore the performance of these methods with different size of the common label set $\mathcal{C}$. The same task is used as before, and we have $|\mathcal{C}|+\left|\overline{\mathcal{C}}_{t}\right|+\left|\overline{\mathcal{C}}_{s}\right|=31.$ For simplicity, we let $\left|\overline{\mathcal{C}}_{t}\right|=\left|\overline{\mathcal{C}}_{s}\right|+1$ or $\left|\overline{\mathcal{C}}_{t}\right|=\left|\overline{\mathcal{C}}_{s}\right|$, and let $|\mathcal{C}|$ vary from 0 to 31. Fig. \ref{c} shows the performance of these methods with different sizes of $|\mathcal{C}|$. When $|\mathcal{C}|=0$, there is no known class appearing in target domain, i.e. $\mathcal{C}_{t} \cap \mathcal{C}_{s}=\emptyset$. We see that the accuracy of S-UAN is similar to UDA’s. When $|\mathcal{C}|>0$, i.e. $\mathcal{C}_{t} \cap \mathcal{C}_{s}\neq\emptyset$, we find that S-UAN outperforms all the compared methods with large margins, revealing that class-wise weighting mechanism of proposed \emph{margin vector} is more general and effective than sample-wise weighting mechanism in UDA.

\begin{figure}
	\centering
	\subfigure[Accuracy w.r.t. $\left|\overline{\mathcal{C}}_{t}\right|$]{
		\begin{minipage}[c]{0.23\textwidth}
			\centering
			\includegraphics[width=1\linewidth]{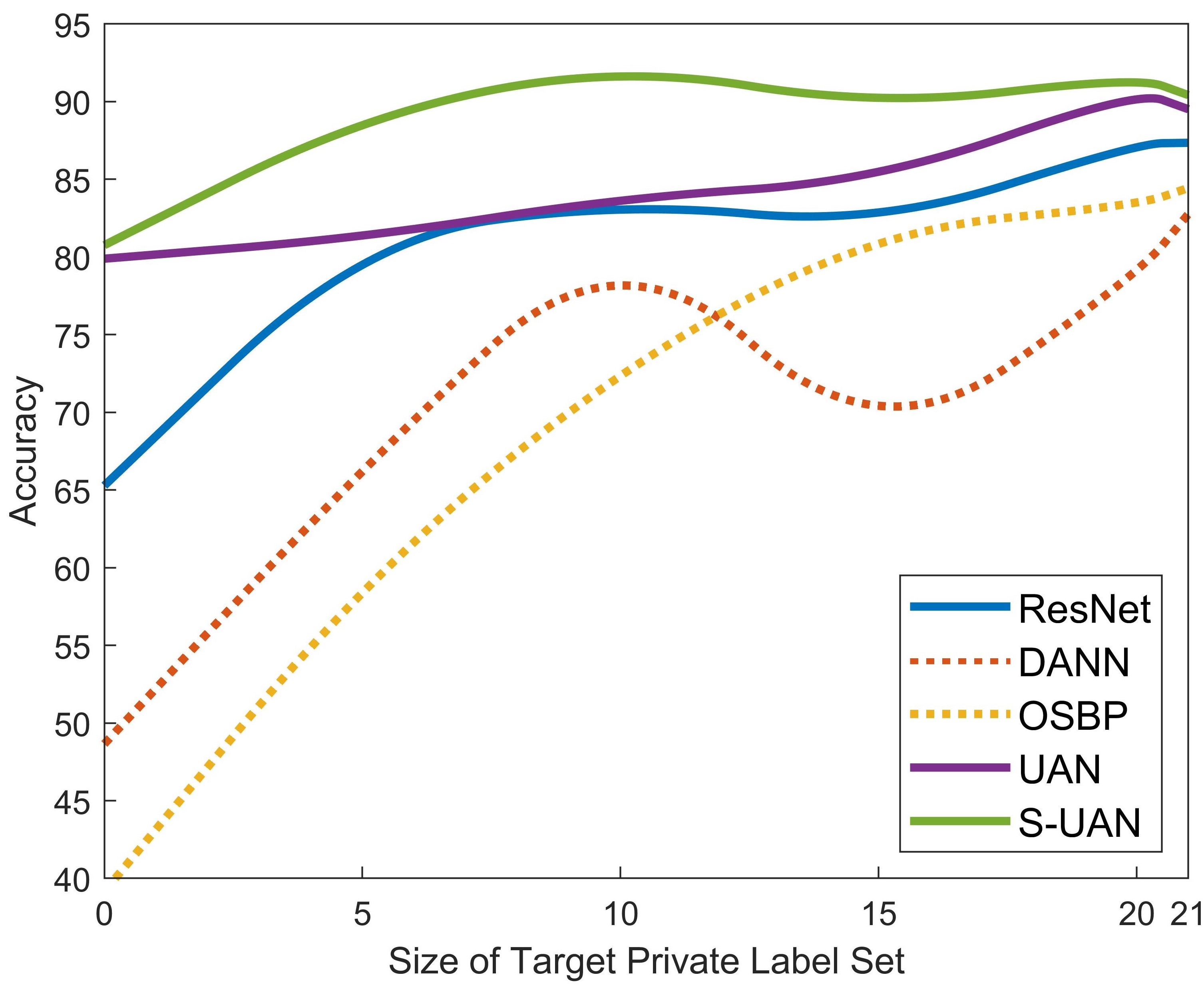}
			\label{ct}
		\end{minipage}}
	\subfigure[Accuracy w.r.t. $|\mathcal{C}|$]{
		\begin{minipage}[c]{0.23\textwidth}
			\centering
			\includegraphics[width=1\linewidth]{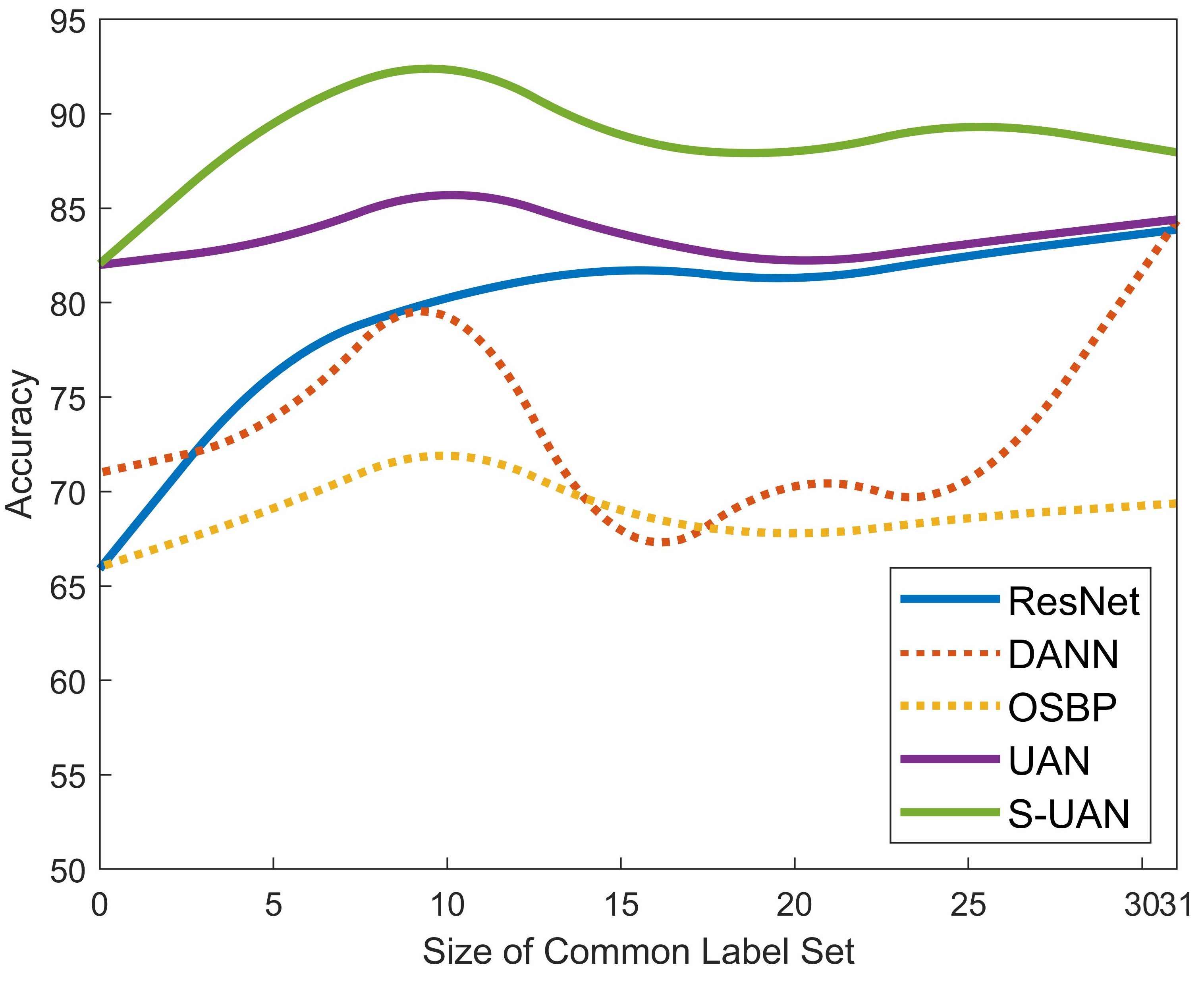}
			\label{c}
		\end{minipage}}
	\caption{Analysis on \textbf{Different UDA Settings}.}
	\label{UDAset}
\end{figure}

\begin{figure}
	\centering
	\subfigure[Accuracy w.r.t. $\left|\mathcal{C}_{t}\right|$]{
		\begin{minipage}[c]{0.23\textwidth}
			\centering
			\includegraphics[width=1\linewidth]{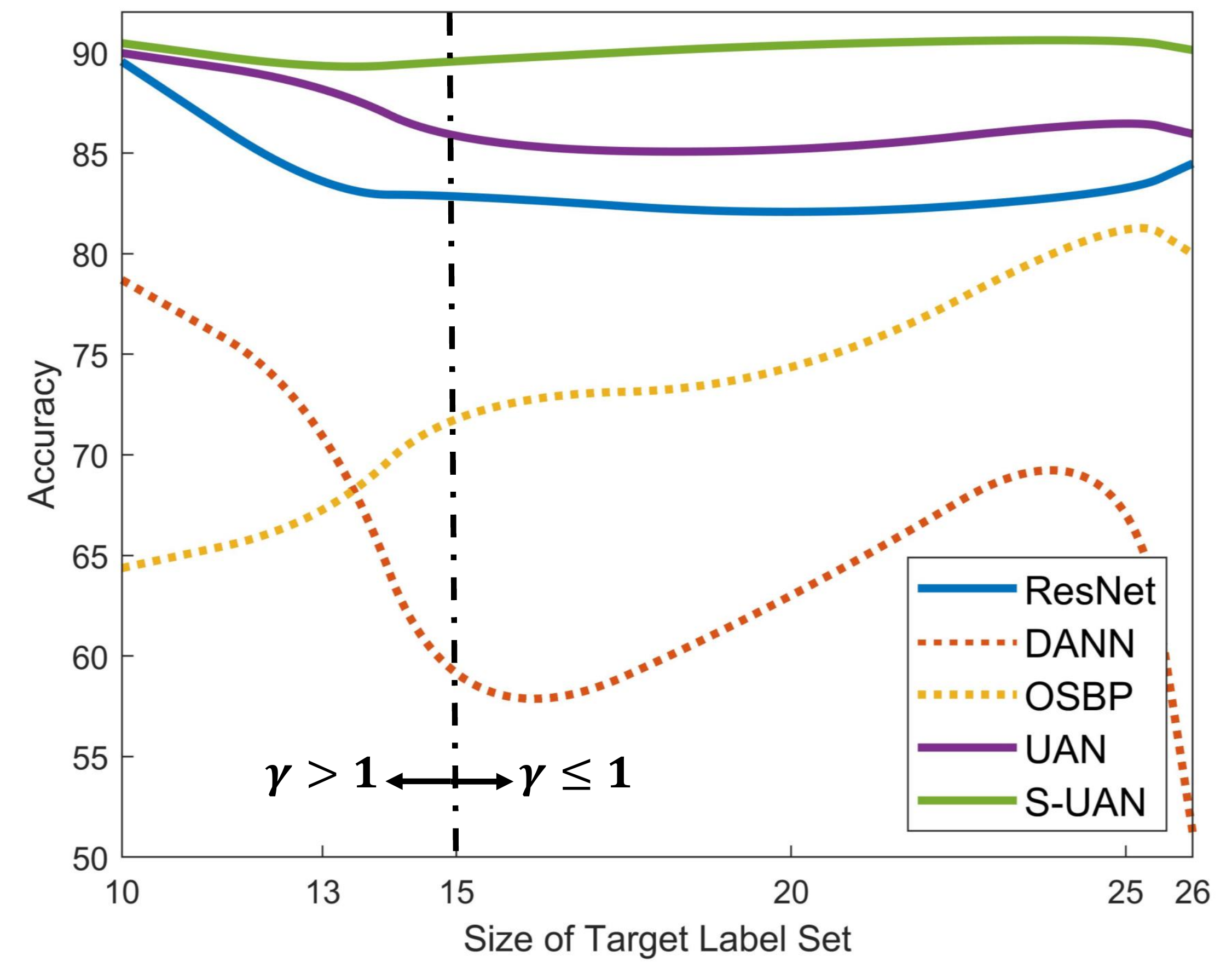}
			\label{bct}
		\end{minipage}}
	\subfigure[Accuracy w.r.t. $|\mathcal{C}|$]{
		\begin{minipage}[c]{0.23\textwidth}
			\centering
			\includegraphics[width=1\linewidth]{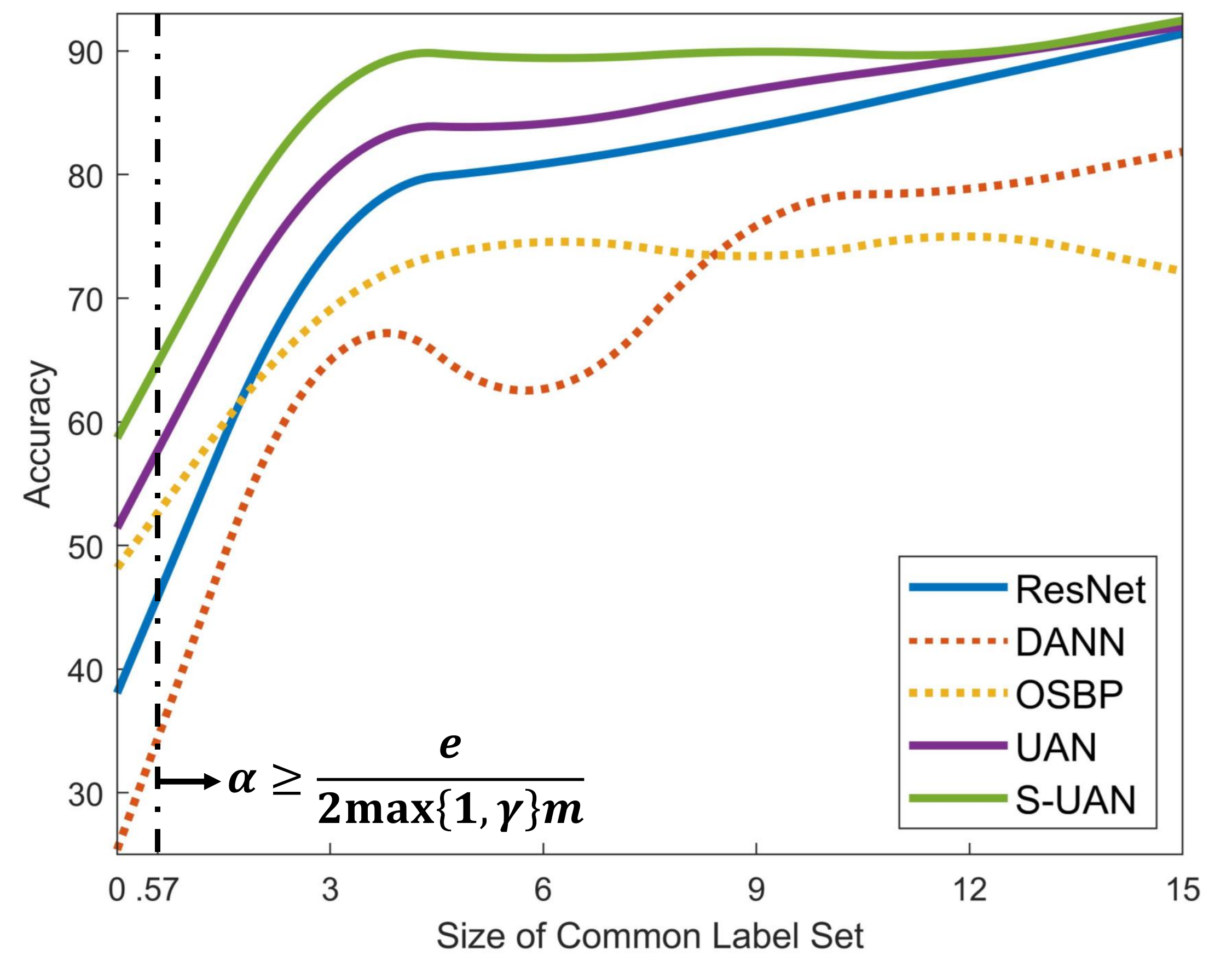}
			\label{bc}
		\end{minipage}}
	\caption{Analysis on \textbf{Target Risk}.}
	\label{GB}
\end{figure}

\subsection{Analysis on Target Risk}
In this section, we verify the properties of generalization bound by analyzing the target risk for task $\mathbf{A}\rightarrow \mathbf{D}$. We fix $|\mathcal{C}_s|=15$. \textbf{ResNet}, \textbf{DANN}, \textbf{OSBP}, \textbf{UAN} and \textbf{S-UAN} are compared in what follows.

\textbf{Property 1: with fixed $|\mathcal{C}|$ and the increase of $|\mathcal{C}_t|$, the generalization bound increases when $|\mathcal{C}_t|<|\mathcal{C}_s|$.} In this experiment, we fix $|\mathcal{C}|=10$ and set $|\mathcal{C}_t|$=10,13,15,20,25,26, correspondingly, $\gamma=1.50,1.15,1.00,0.75,0.60,0.58$. The experiment results are shown in Fig. \ref{bct}, except DANN and OSBP, other well-behaved methods are all suffered from performance degradation with the increase of $|\mathcal{C}_t|$ when $\gamma>1$. And when $\gamma\leq1$, the performance of three well-behaved methods does not vary obviously. This experimental result verifies the first property of Theorem \ref{theorem}. And we can find that the target risk of the proposed S-UAN is the lowest, which is the best learning algorithm for universal domain adaptation compared with the state-of-the-art domain adaptation methods.

\textbf{Property 2: with the increase of $|\mathcal{C}|$, the generalization bound decreases.} In this experiment, we fix $|\mathcal{C}_t|=15$ and set $|\mathcal{C}|=0,3,6,9,12,15$, correspondingly, $\alpha=0.0,0.2,0.4,0.6,0.8,1.0$, here $\gamma=1$. The experiment results are shown in Fig. \ref{bc}, the sampling batch size $m=36$, we can see that the performance of all the methods is improved with the increase of $|\mathcal{C}|$ when $|\mathcal{C}|=\alpha|\mathcal{C}_s|\geq\frac{e}{2\max\{1,\gamma\}m}|\mathcal{C}_s|=0.57$. Note that $|\mathcal{C}|$ can only be an integer, although the generalization bound becomes higher when $|\mathcal{C}|$ varies from 0 to 0.57, we find that the value of generalization bound at $|\mathcal{C}|=1$ is lower than $|\mathcal{C}|=0$. This experimental result verifies the second property of Theorem \ref{theorem}. And the target risk of the proposed S-UAN is still the lowest of these methods.

\subsection{Analysis of S-UAN}
\label{Analysis of S-UAN}
\textbf{Weight Analysis}. To verify the inference introduced in Section \ref{Introduction}, we plot the estimated probability density distribution for different components of the $w^s(\mathbf{x})$ and $w^t(\mathbf{x})$ in Fig. \ref{wd}. Results show that the distributions are well separated between shared and private label sets of both domains, explaining why S-UAN achieves stable performance in different UDA settings. We observe that distributions of source shared and private parts are much more distinguishable compared with UAN's. And distributions of target shared and private parts are as distinguishable as UAN's. Note that the definition of the weight of target samples is simple in S-UAN, but is also distinguishable enough for target shared and private parts.

\begin{figure}
	\centering
	\subfigure[A$\rightarrow$D]{
		\begin{minipage}[c]{0.23\textwidth}
			\centering
			\includegraphics[width=1\linewidth]{figures/w-A-D.png}
			\label{src}
		\end{minipage}}
	\subfigure[Ar$\rightarrow$Cl]{
		\begin{minipage}[c]{0.23\textwidth}
			\centering
			\includegraphics[width=1\linewidth]{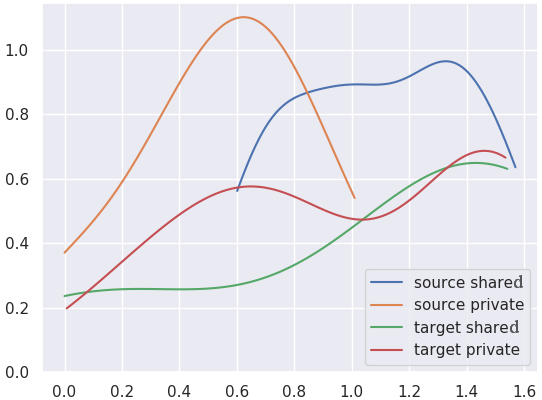}
			\label{adv}
		\end{minipage}}
	\subfigure[VisDA]{
		\begin{minipage}[c]{0.23\textwidth}
			\centering
			\includegraphics[width=1\linewidth]{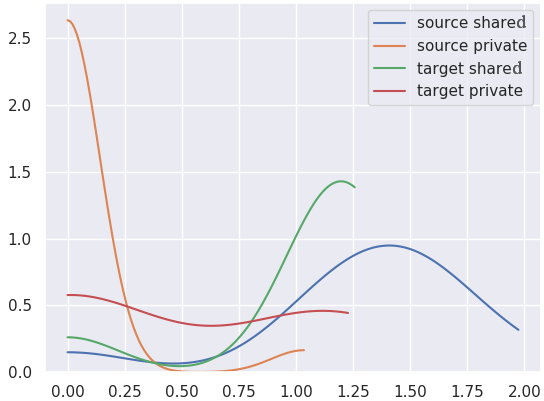}
			\label{metric}
		\end{minipage}}
	\subfigure[ImageNet$\rightarrow$Caltech]{
		\begin{minipage}[c]{0.23\textwidth}
			\centering
			\includegraphics[width=1\linewidth]{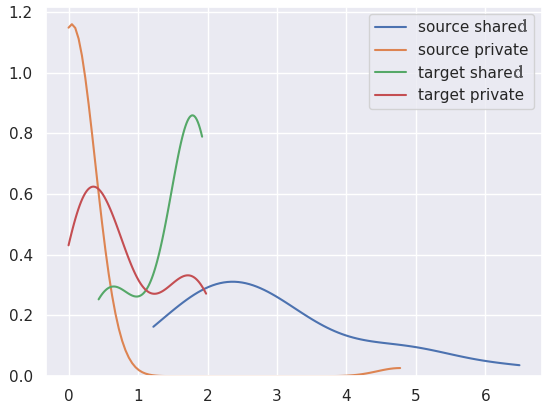}
			\label{metric}
		\end{minipage}}
	\caption{\textbf{Probability density distribution} of the $w^s(\mathbf{x})$ $($or $w^t(\mathbf{x})$$)$ on the following groups: (1) 'source shared': source samples in $\mathcal{C}$ (blue); (2) 'source private': source samples in $\overline{\mathcal{C}}_{s}$ (orange); (3) 'target shared': target samples in $\mathcal{C}$ (green); (4) 'target private': target samples in $\overline{\mathcal{C}}_{t}$ (red).}
	\label{wd}
\end{figure}

\textbf{Threshold Sensitivity}. Although we have fixed the criterion for distinguishing unknown classes in all experiments, i.e. $w^{t}(\mathbf{x}) \geq 0.5$. We also explore the sensitivity of the threshold for analysis. As shown in Fig. \ref{w0}, though S-UAN's accuracies vary by $2\sim3\%$ with the change of the threshold, it achieves stable performance and outperforms other methods by large margins with the change of the threshold. Note that in our experiments, UAN's accuracy on VisDA can't reach their report in our experiments. However, S-UAN's accuracies still outperform UAN's reported accuracy, which is fully tuned and their best accuracy.

\begin{figure}
	\centering
	\subfigure[A$\rightarrow$D]{
		\begin{minipage}[c]{0.23\textwidth}
			\centering
			\includegraphics[width=1\linewidth]{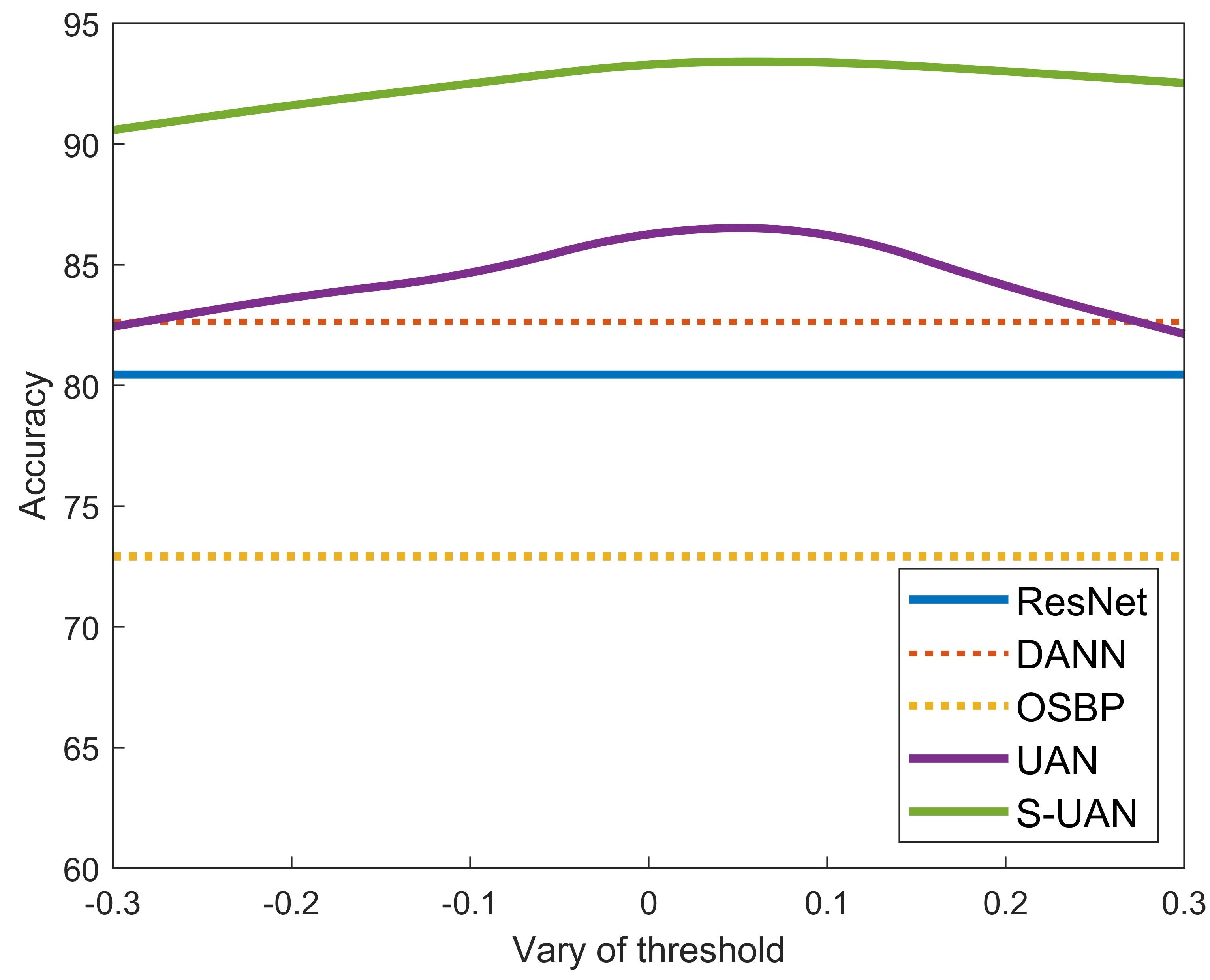}
		\end{minipage}}
	\subfigure[VisDA]{
		\begin{minipage}[c]{0.23\textwidth}
			\centering
			\includegraphics[width=1\linewidth]{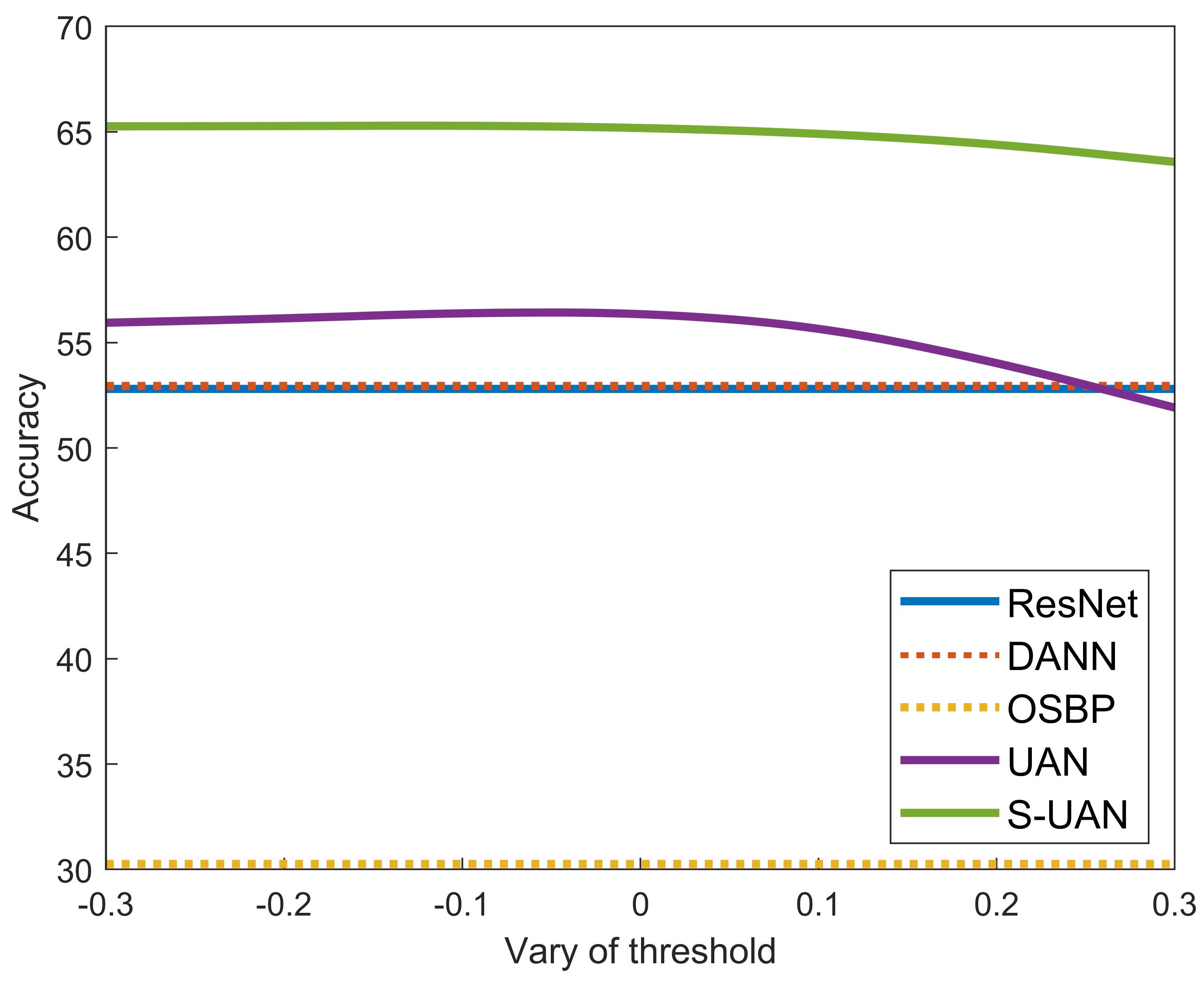}
		\end{minipage}}
	\caption{\textbf{Sensitivity to threshold}. The horizontal axis is the value added to the initial threshold, i.e. $0.5\pm0.3$ in S-UAN and $-0.5\pm0.3$ in UAN.}
	\label{w0}
\end{figure}

\textbf{Hyperparameter Analysis}. The only hyperparameter in S-UAN is the activated threshold $w_0$ of the normalized function as shown in Eq.\ref{norm}. The value of $w_0$ can be 0 or 1. When $w_0=0$, $w_0$ has no effect on Eq.\ref{norm}, and all the $w^s(\mathbf{x})$ $($or $w^t(\mathbf{x})$$)$ can be activated. When $w_0=1$, only the $w^s(\mathbf{x})$ $($or $w^t(\mathbf{x})$$)$ greater than the average value can be activated, and the remaining $w^s(\mathbf{x})$ $($or $w^t(\mathbf{x})$$)$ are set to 0. In the default settings of S-UAN, $w_0$ is set to 1 in Office-31 and VisDA, where $|\mathcal{C}_s|\geq |\mathcal{C}_t|$. And $w_0$ is set to 0 in Office-Home and ImageNet-Caltech, where $|\mathcal{C}_s|<|\mathcal{C}_t|$. To do \textbf{ablation study} on $w_0$, we also set $w_0=0$ in Office-31 and VisDA (\textbf{S-UAN $\mathbf{w_0=0}$}) in Table \ref{Office-31}. We can find that S-UAN outperforms all the compared methods, even if $w_0=0$ in all experiments.

\section{Related Work}
In this section, we briefly review DA methods and margin theory in classification.
\subsection{Universal Domain Adaptation}
Universal Domain Adaptation was first proposed in \cite{you2019universal}, where the target label set is unknown. In traditional domain adaptation settings, the target label set or its relationship with source label set is assumed to be known. These traditional settings can be summarized by three types: closed set, partial and open set DA. Closed set DA assumes that $\mathcal{C}_s=\mathcal{C}_t$, where the domain shift is eliminated as much as possible \cite{DBLP:conf/eccv/SaenkoKFD10,DBLP:journals/pami/DuanTX12,DBLP:conf/icml/ZhangSMW13}. These methods provide insight in developing deep adaptation methods \cite{DBLP:journals/corr/TzengHZSD14,DBLP:conf/icml/LongC0J15,DBLP:journals/jmlr/GaninUAGLLML16,DBLP:conf/iccv/HausserFMC17,long2016unsupervised,DBLP:conf/cvpr/TzengHSD17,DBLP:conf/cvpr/SaitoWUH18,long2018conditional,8890009,8956072,8475034,8449326,8963871}. And some other deep adaptation methods explore architecture designs \cite{DBLP:conf/iccv/CarlucciPCRB17,DBLP:conf/icml/XieZCC18,DBLP:conf/cvpr/MurezKKRK18,DBLP:conf/cvpr/VolpiMSM18,DBLP:conf/cvpr/HuKSC18,DBLP:conf/cvpr/ChenLWWC18,9108549,8836530,8475006}. Meanwhile, with the proposal of Generative Adversarial Nets (GAN) \cite{DBLP:conf/nips/GoodfellowPMXWOCB14}, GAN-based DA methods are proposed \cite{DBLP:journals/jmlr/GaninUAGLLML16,DBLP:conf/cvpr/TzengHSD17,ganin2016domain,8698453,8833506}. In these methods, a discriminator $D$ is designed to distinguish source and target features, meanwhile, the feature extractor $F$ is encouraged to generate domain-invariant features which confuse $D$. Overall, the adversarial loss is usually defined as
\begin{equation}
\begin{aligned} 
E_{D}=&-\mathbb{E}_{\mathbf{x} \sim p}(\mathbf{x}) \log D(F(\mathbf{x}))
\\ &-\mathbb{E}_{\mathbf{x} \sim q}(\mathbf{x}) \log (1-D(F(\mathbf{x}))).
\label{ld}
\end{aligned} 
\end{equation}

We can see that the only difference between Eq. \ref{ED} and Eq. \ref{ld} is the sample-wise weighting mechanism. And the ideal distribution of $w^s(\mathbf{x})$ $($or $w^t(\mathbf{x})$$)$ mentioned in Section \ref{Introduction} can degenerate universal domain adaptation problem by
\begin{equation}
\begin{aligned} 
E_{D}=&-\mathbb{E}_{\mathbf{x} \sim p}w^{s*}(\mathbf{x}) \log D(F(\mathbf{x}))
\\ &-\mathbb{E}_{\mathbf{x} \sim q}w^{t*}(\mathbf{x}) \log (1-D(F(\mathbf{x})))
\\=&-\mathbb{E}_{\mathbf{x} \sim p_c} \log D(F(\mathbf{x}))
\\ &-\mathbb{E}_{\mathbf{x} \sim q_c} \log (1-D(F(\mathbf{x}))),
\label{dg}
\end{aligned}
\end{equation}where $w^{s*}(\mathbf{x})$ and $w^{t*}(\mathbf{x})$ is the ideal $w^s(\mathbf{x})$ $($or $w^t(\mathbf{x})$$)$. A well-defined weight mechanism ensures further domain adaptation in $\mathcal{C}$. This is one of the basic solution for UDA.

Partial DA assumes that the target label set can be contained in the souce label set \cite{DBLP:conf/cvpr/CaoL0J18,DBLP:conf/cvpr/0017DLO18,cao2018partialeccv}. Obviously, this setting is not practical in the wild. Open set DA assumes that $\mathcal{C}_s\subset\mathcal{C}_t$ \cite{saito2018open,panareda2017open}. And in \cite{panareda2017open}, the classes private to each domain are totally consider as an “unknown” class. This setting is more practical than partial domain adaptation, but still hold that the classes private to each domain exist. Therefore, a general scenario of domain adaptation is established in \cite{you2019universal}, called universal domain adaptation.

\subsection{Margin Theory for Classification}
In the early study of classification problems, maximum margin techniques have been used to identify relevant-irrelevant label pairs \cite{DBLP:conf/nips/ElisseeffW01} or output coding margins \cite{DBLP:conf/aaai/LiuT15a}. And a margin theory for classification was developed by \cite{Vladimir2000Empirical}, where the 0-1 loss for classification is replaced by developed margin loss. Recently, \cite{8680669} maximize margins between relevant-relevant label pairs with different importance degrees. In \cite{DBLP:conf/aaai/JiaZ20}, a maximum margin multi-dimensional classification is investigated by leveraging classification margin maximization on individual class variable and modeling relationship regularization across class variables. Based on the margin theory introduced in \cite{Vladimir2000Empirical}, \cite{zhang2019bridging} introduced the margin loss to domain adaptation scenarios, achieving better performance than the use of traditional 0-1 loss.

In our previous work \cite{DBLP:journals/corr/abs-2004-10963}, the margin of classification is used to adaptively adjust the margin of triplet loss. In this work, the margin of classification is used to find known classes in the target label set. And a novel \emph{margin vector} is firstly introduced to universal domain adaptation setting.

\section{Conclusion}
In this paper, we employ a probabilistic method for locating the common label set, where each source class may come from the common label set with a probability. In particular, we propose a novel \emph{margin vector} for evaluating the probability of each source class from the common label set, which can estimate the probability accurately. We further propose a target margin register to enable \emph{margin vector} to update continuously during iterative training. This technology can solve UDA problem efficiently. Then we propose a simple universal adaptation network (S-UAN), which is more simple but efficient than existing UDA methods. A comprehensive evaluation shows that S-UAN works well in general UDA setting and outperforms the state-of-the-art methods by large margins.

\section*{Acknowledgments}
This work was supported in part by the National Natural Science Foundation of China under Grant 61671252, 61571233 and 61901229; the Natural Science Research of Higher Education Institutions of Jiangsu Province under Grant 19KJB510008.

\bibliographystyle{IEEEtran}
\bibliography{S-UAN}

\end{document}